\documentclass{article}

\usepackage[final]{neurips_2021}

\usepackage[utf8]{inputenc} 
\usepackage[T1]{fontenc}    
\usepackage{hyperref}       
\usepackage{url}            
\usepackage{booktabs}       
\usepackage{amsfonts}       
\usepackage{nicefrac}       
\usepackage{microtype}      
\usepackage{xcolor}         

\usepackage{tabularx}
\usepackage{graphicx}
\usepackage{hyperref}\usepackage{natbib}
\usepackage{delarray}
\usepackage{times}
\usepackage{subfigure}
\usepackage{color}
\usepackage{amssymb}
\usepackage{amsmath}
\usepackage{amsthm}
\usepackage{mathrsfs}
\usepackage{tikz}
\usepackage{xspace}
\usepackage{enumerate}
\usepackage{epstopdf}

\usepackage[normalem]{ulem}

\usepackage{algorithm}
\usepackage{algorithmic}
\usepackage{nameref}
\usepackage{booktabs}
\usepackage{mdframed}
\usepackage{fmtcount}
\usepackage{footnote}
\usepackage{bbm}
\usepackage{thmtools,thm-restate}
\declaretheorem{theorem}

\newcommand{\poly}{{\text{poly}}}
\newcommand{\eat}[1]{}
\newcommand{\R}{{\mathbb R}}

\newcommand{\E}{{\mathbb{E}}}

\newcommand{\inner}[2]{\left\langle #1, #2 \right\rangle}
\newcommand{\ns}[1]{\left\| #1 \right\|^2}
\newcommand{\n}[1]{\left\| #1 \right\|}
\newcommand{\fns}[1]{\left\| #1 \right\|^2_F}
\newcommand{\fn}[1]{\left\| #1 \right\|_F}

\newcommand{\indic}[1]{\mathbbm{1}\left\{#1\right\}}

\newcommand{\pr}[1]{\left( #1 \right)}
\newcommand{\br}[1]{\left[ #1 \right]}

\newcommand{\absr}[1]{\left| #1 \right|}

\newcommand{\colW}{\text{col}(W)}

        {\hspace*{\fill}$\Box$\par\vspace{4mm}}
\newenvironment{proofof}[1]{\smallskip\noindent{\bf Proof of #1.}}%
        {\hspace*{\fill}$\Box$\par}

\newtheorem{lemma}[theorem]{Lemma}

\newtheorem{definition}{Definition}
\newtheorem{assumption}{Assumption}

\newtheorem{prop}{Proposition}

\usepackage{apptools}
\AtAppendix{\counterwithin{lemma}{section}}
\AtAppendix{}
\AtAppendix{}
\AtAppendix{\counterwithin{theorem}{section}}
\AtAppendix{}
\AtAppendix{}

\hypersetup{
    colorlinks=true,       
    linkcolor=black,    
    citecolor=black,        
    filecolor=magenta,     
    urlcolor=blue
}

\newcommand{\bv}{\bar{v}}
\newcommand{\bw}{\bar{w}}
\newcommand{\wst}{w^{(s,t)}}
\newcommand{\vst}{v^{(s,t)}}
\newcommand{\bwst}{\bar{w}^{(s,t)}}
\newcommand{\bvst}{\bar{v}^{(s,t)}}
\newcommand{\vt}{v^{(t)}}
\newcommand{\wt}{w^{(t)}}
\newcommand{\bvt}{\bar{v}^{(t)}}
\newcommand{\bvO}{\bar{v}^{(0)}}
\newcommand{\bwt}{\bar{w}^{(t)}}
\newcommand{\Wst}{W^{(s,t)}}
\newcommand{\Wt}{W^{(t)}}
\newcommand{\Tst}{T^{(s,t)}}
\newcommand{\Tt}{T^{(t)}}
\newcommand{\hast}{\hat{a}^{(s,t)}}
\newcommand{\Sst}{S^{(s,t)}}
\newcommand{\Est}{\E^{(s,t)}}
\newcommand{\hatt}{\hat{a}^{(t)}}
\newcommand{\St}{S^{(t)}}
\newcommand{\Et}{\E^{(t)}}

\newcommand{\ha}{\hat{a}}
\newcommand{\ta}{\tilde{a}}
\newcommand{\tat}{\tilde{a}^{(t)}}
\newcommand{\Deltat}{\Delta^{(t)}}
\newcommand{\ninf}[1]{\left\| #1 \right\|_{\infty}}
\newcommand{\rd}{\mathrm{d}}
\newcommand{\pt}{p^{(t)}}
\newcommand{\pO}{p^{(0)}}
\newcommand{\qt}{q^{(t)}}
\newcommand{\qO}{q^{(0)}}

\newcommand{\dbprime}{{\prime\prime}}
\newcommand{\triprime}{{\prime\prime\prime}}
\newcommand{\bvx}[1]{\bar{v}^{(#1)}}
\newcommand{\vx}[1]{v^{(#1)}}
\newcommand{\polylog}{\mathrm{polylog}}
\newcommand{\Ss}{S^{(s)}}

\newcommand{\eps}{\varepsilon}
\newcommand{\zt}{z^{(t)}}
\newcommand{\xt}{x^{(t)}}

\newenvironment{remark}[1][Remark]
  {\begin{proof}[\textnormal{\textbf{#1}}]
    
  }
  {\end{proof}}
\newsavebox{\mybox}

\makeatletter
\newcommand{\printfnsymbol}[1]{%
  \textsuperscript{\@fnsymbol{#1}}%
}
\makeatother

\title{Understanding Deflation Process in Over-parametrized Tensor Decomposition}

\author{%
    Rong Ge\thanks{Alphabetical order.} \\
    Duke University \\
    \texttt{rongge@cs.duke.edu} \\
    \And
    Yunwei Ren\printfnsymbol{1} \\
    Shanghai Jiao Tong University \\
    \texttt{2016renyunwei@sjtu.edu.cn} \\
    \AND
    \hspace{-5mm}Xiang Wang\printfnsymbol{1} \\
    \hspace{-6.5mm}Duke University \\
    \hspace{-5mm}\texttt{xwang@cs.duke.edu} \\
    \And
    Mo Zhou\printfnsymbol{1} \\
    Duke University \\
    \texttt{mozhou@cs.duke.edu} \\
}

\begin{document}

\maketitle

\begin{abstract}
In this paper we study the training dynamics for gradient flow on over-parametrized tensor decomposition problems. Empirically, such training process often first fits larger components and then discovers smaller components, which is similar to a tensor deflation process that is commonly used in tensor decomposition algorithms. We prove that for orthogonally decomposable tensor, a slightly modified version of gradient flow would follow a tensor deflation process and recover all the tensor components. Our proof suggests that for orthogonal tensors, gradient flow dynamics works similarly as greedy low-rank learning in the matrix setting, which is a first step towards understanding the implicit regularization effect of over-parametrized models for low-rank tensors.
\end{abstract}

\section{Introduction}

Recently, over-parametrization has been recognized as a key feature of neural network optimization. A line of works known as the Neural Tangent Kernel (NTK) showed that it is possible to achieve zero training loss when the network is sufficiently over-parametrized \citep{jacot2018neural,du2018gradient,allen2018convergence}. However, the theory of NTK implies a particular dynamics called lazy training where the neurons do not move much \citep{chizat2019lazy}, which is not natural in many settings and can lead to worse generalization performance~\citep{arora2019exact}. Many works explored other regimes of over-parametrization \citep{chizat2018global,mei2018mean} and analyzed dynamics beyond lazy training \citep{allen2018learning,li2020learning,wang2020beyond}.

Over-parametrization does not only help neural network models. In this work, we focus on a closely related problem of tensor (CP) decomposition. In this problem, we are given a tensor of the form
\[
T^* = \sum_{i=1}^r a_i (U[:,i])^{\otimes 4},
\]
where $a_i\geq 0$ and $U[:,i]$ is the $i$-th column of $U\in \R^{d\times r}$. The goal is to fit $T^*$ using a tensor $T$ of a similar form:
\begin{equation*}
    T = \sum_{i=1}^m \frac{(W[:,i])^{\otimes 4}}{\|W[:,i]\|^2}. 
\end{equation*}
Here $W$ is a $d\times m$ matrix whose columns are components for tensor $T$. The model is over-parametrized when the number of components $m$ is larger than $r$. The choice of normalization factor of $1/\|W[:,i]\|^2$ is made to accelerate gradient flow (similar to \citet{li2020learning,wang2020beyond}). 

Suppose we run gradient flow on the standard objective $\frac{1}{2}\|T-T^*\|_F^2$, that is, we evolve $W$ according to the differential equation:
\[
\frac{\rd W}{\rd t} = - \nabla\pr{ \frac{1}{2}\|T-T^*\|_F^2},
\]
can we expect $T$ to fit $T^*$ with good accuracy? Empirical results (see Figure~\ref{fig:ortho}) show that this is true for orthogonal tensor $T^*$\footnote{We say $T^*$ is an orthogonal tensor if the ground truth components $U[:,i]$'s are orthonormal.} as long as $m$ is large enough. Further, the training dynamics exhibits a behavior that is similar to a {\em tensor deflation process}: it finds the ground truth components one-by-one from larger component to smaller component (if multiple ground truth components have similar norm they might be found simultaneously).


\begin{figure}[t]
\subfigure{
    \includegraphics[width=2.65in]{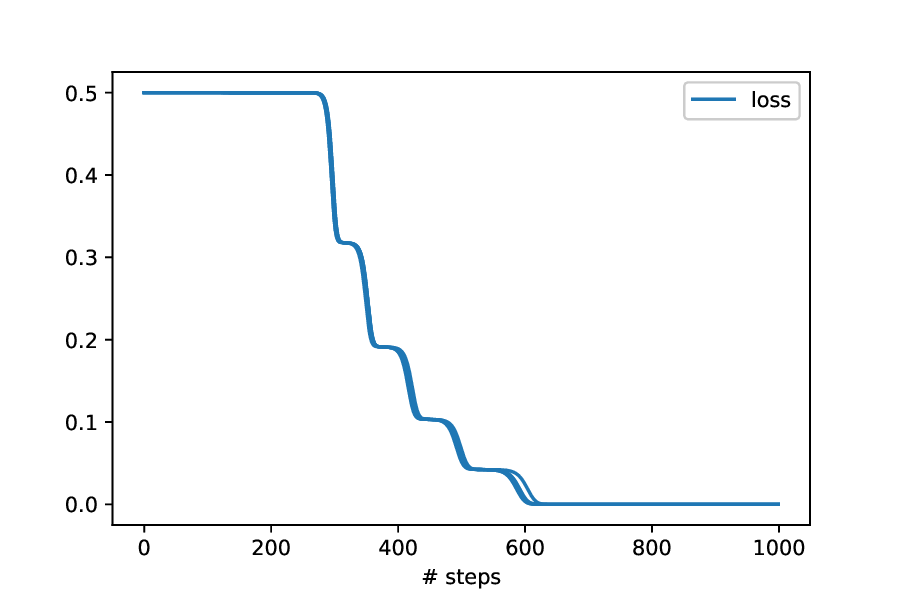}
}
\subfigure{
    \includegraphics[width=2.65in]{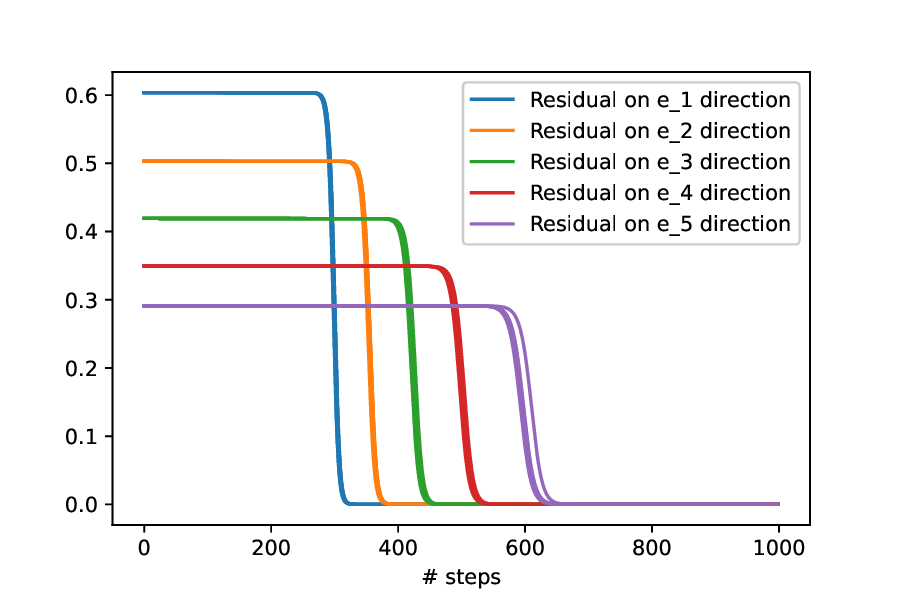}
}
\vspace{-0.2cm}
\caption{The training trajectory of gradient flow on orthogonal tensor decompositions. We chose $T^* = \sum_{i\in [5]}a_i e_i^{\otimes 4}$ with $e_i \in \R^{10}$ and $a_i/a_{i+1}=1.2.$ Our model $T$ has $50$ components and each component is randomly initialized with small norm $10^{-15}$. We ran the experiments from $5$ different initialization and plotted the results separately. The left figure shows the loss $\frac{1}{2}\fns{T-T^*}$ and the right figure shows the residual on each $e_i$ direction that is defined as $(T^*-T)(e_i^{\otimes 4}).$ }\label{fig:ortho}
\vspace{-0.5cm}
\end{figure}

In this paper we show that with a slight modification, gradient flow on over-parametrized tensor decomposition is guaranteed to follow this tensor deflation process, and can fit any orthogonal tensor to desired accuracy\footnotemark (see Section~\ref{sec:algorithm} for the algorithm and Theorem~\ref{thm:main} for the main theorem). This shows that for orthogonal tensors, the trajectory of modified gradient-flow is similar to a greedy low-rank process that was used to analyze the implicit bias of low-rank matrix factorization~\citep{li2020towards}. 
We emphasize that our goal is not to propose another tensor 
decomposition algorithm. Instead, we hope our results can serve as a first step in understanding the implicit bias 
of over-parameterized gradient descent for low-rank tensor 
problems.
\footnotetext{Due to some technical challenges, we actually require 
the target accuracy to be at least $\exp(-o(d / \log d))$. This is only a 
very mild restriction since the dependence is exponential in $d$, and in 
practice, $d$ is usually large and this lower bound can easily drop below the
numerical precision.
}

\subsection{Our approach and technique}

To understand the tensor deflation process shown in Figure~\ref{fig:ortho}, intuitively we can think about the discovery and fitting of a ground truth component in two phases. Consider the beginning of the gradient flow as an example. Initially all the components in $T$ are small, which makes $T$ negligible compared to $T^*$. In this case each component $w$ in $W$ will evolve according to a simpler dynamics that is similar to tensor power method, where one updates $w$ to $T^*(w^{\otimes 3},I)/\n{T^*(w^{\otimes 3},I)}$ (see Section~\ref{sec:deflation} for details).

For orthogonal tensors, it's known that tensor power method with random initializations would be able to discover the largest ground truth components (see \cite{anandkumar2014tensor}). Once the largest ground truth component has been discovered, the corresponding component (or multiple components) $w$ will quickly grow in norm, which eventually fits the ground truth component. The flat regions in the trajectory in Figure~\ref{fig:ortho} correspond to the period of time where the components $w$'s are small and $T-T^*$ remains stable, while the decreasing regions correspond to the period of time where a ground truth component is being fitted. 

However, there are many challenges in analyzing this process. The main problem is that the gradient flow would introduce a lot of dependencies throughout the trajectory, making it harder to analyze the fitting of later ground truth components, especially ones that are much smaller. We modify the algorithm to include a reinitialization step per epoch, which alleviates the dependency issue. Even after the modification we still need a few more techniques:

\paragraph{Local stability} One major problem in analyzing the dynamics in a later stage is that the components used to fit the previous ground truth components are still moving according to their gradients, therefore it might be possible for these components to move away. To address this problem, we add a small regularizer to the objective, and give a new local stability analysis that bounds the distance to the fitted ground truth component both individually and on average. The idea of bounding the distance on average is important as just assuming each component $w$ is close enough to the fitted ground truth component is not sufficient to prove that $w$ cannot move far. While similar ideas were considered in \cite{chizat2021sparse}, the setting of tensor decomposition is different.

\paragraph{Norm/Correlation relation} A key step in our analysis establishes a relationship between norm and correlation: we show if a component $w$ crosses a certain norm threshold, then it must have a very large correlation with one of the ground truth components. This offers an initial condition for local stability and makes sure the residual $T^*-T$ is almost close to an orthogonal tensor. Establishing this relation is difficult as unlike the high level intuition, we cannot guarantee $T^*-T$ remains unchanged even within a single epoch: it is possible that one ground truth component is already fitted while no large component is near another ground truth component of same size. In previous work, \cite{li2020learning} deals with a similar problem for neural networks using gradient truncation that prevents components from growing in the first phase (and as a result has super-exponential dependency on the ratio between largest and smallest $a_i$). We give a new technique to control the influence of ground truth components that are fitted within this epoch, so we do not need the gradient truncation and can characterize the deflation process.

\subsection{Related works}

\paragraph{Neural Tangent Kernel}
There is a recent line of work showing the connection between Neural Tangent Kernel (NTK) and sufficiently wide neural networks trained by gradient descent \citep{jacot2018neural,allen2018convergence,du2018gradient,du2019gradient,li2018learning,arora2019exact,arora2019fine,zou2020gradient,oymak2020towards,ghorbani2021linearized}. These papers show when the width of a neural network is large enough, it will stay around the initialization and its training dynamic is close to the dynamic of the kernel regression with NTK. In this paper we go beyond the NTK setting and analyze the trajectory from a very small initialization.

\paragraph{Mean-field analysis} There is another line of works that use mean-field approach to study the optimization for infinite-wide neural networks \citep{mei2018mean,chizat2018global,nguyen2020rigorous,nitanda2017stochastic,wei2019regularization,rotskoff2018trainability,sirignano2020mean}. \cite{chizat2019lazy} showed that, unlike NTK regime, the parameters can move away from its initialization in mean-field regime. However, most of the existing works need width to be exponential in dimension and do not provide a polynomial convergence rate.

\paragraph{Beyond NTK}
There are many works showing the gap between neural networks and NTK \citep{allen2019can,allen2018learning,yehudai2019power,ghorbani2019limitations,ghorbani2020neural,dyer2019asymptotics,woodworth2020kernel,bai2019beyond,bai2020taylorized,huang2020dynamics,chen2020towards}. In particular, \cite{li2020learning} and \cite{wang2020beyond} are closely related with our setting. While \cite{li2020learning} focused on learning two-layer ReLU neural networks with orthogonal weights, they relied on the connection between tensor decomposition and neural networks \citep{ge2017learning} and essentially worked with tensor decomposition problems. In their result, all the $a_i$'s are within a constant factor and all components are learned simultaneously. We allow ground truth components with very different scale and show a deflation phenomenon. 
\cite{wang2020beyond} studied learning a low-rank non-orthogonal 
tensor, but they only showed the learned tensor $T$ will eventually be close to 
the ground truth tensor $T^*$ and does not guarantee the components of $T$ will
align with the components of $T^*$. On the other hand, we fully characterize 
the training trajectory and the components of the learned tensor.

\paragraph{Implicit regularization} 
Many works recently showed that different optimization methods tend to converge to different optima and have different optimization trajectories in several settings \citep{Saxe14exactsolutions, soudry2018implicit, nacson2019convergence, ji2018gradient, ji2018risk, ji2019refined, ji2020directional, gunasekar2018characterizing, gunasekar2018implicit, moroshko2020implicit, arora2019implicit, lyu2019gradient, chizat2020implicit}. 
In particular, \cite{Saxe14exactsolutions} related the dynamics of 
gradient descent to the magnitude of the singular values of the target weight 
matrices for linear networks with orthogonal inputs. The phenomenon there is
qualitatively similar to our results, but the settings and the proof techniques
are very different. The more related and recent works are \cite{li2020towards}
and \cite{razin2021implicit}.
\cite{li2020towards} studied matrix factorization problem and showed gradient descent with infinitesimal initialization is similar to greedy low-rank learning, which is a multi-epoch algorithm that finds the best approximation within the rank constraint and relax the constraint after every epoch. 
\cite{razin2021implicit} studied the tensor factorization problem and showed that it biases towards low rank tensor. Both of these works considered partially observable matrix or tensor and are only able to fully analyze the first epoch (i.e., recover the largest direction). We focus on a simpler setting with fully-observable ground truth tensor and give a complete analysis of learning all the ground truth components.

\subsection{Outline}
In Section~\ref{sec:prelim} we introduce the basic notations and problem setup. In Section~\ref{sec:deflation} we review tensor deflation process and tensor power method. We then give our algorithm in Section~\ref{sec:algorithm}. Section~\ref{sec:sketch} gives the formal main theorem and discusses high-level proof ideas. We conclude in Section~\ref{sec:conclude} and discuss some limitations of the work. The detailed proofs and additional experiments are left in the appendix.

\section{Preliminaries}\label{sec:prelim}

\paragraph{Notations} We use upper-case letters to denote matrices and tensors, and lower-case letters to denote vectors. For any positive integer $n,$ we use $[n]$ to denote the set $\{1,2,\cdots, n\}.$ We use $I_d$ to denote $d\times d$ identity matrix, and omit the subscript $d$ when the dimension is clear. We use $\delta_0$Unif$(\mathbb{S}^{d-1})$ to denote the uniform distribution over $(d-1)$-dimensional sphere with radius $\delta_0.$

For vector $v$, we use $\|v\|$ to denote its $\ell_2$ norm. We use $v_k$ to denote the $k$-th entry of vector $v$, and use $v_{-k}$ to denote vector $v$ with its $k$-th entry removed. We use $\bv$ to denote the normalized vector $\bv=v/\n{v}$, and use $\bv_k$ to denote the $k$-th entry of $\bv.$

For a matrix $A$, we use $A[:,i]$ to denote its $i$-th column and $\mathrm{col}(A)$ to denote the set of all column vectors of $A$. For matrix $M$ or tensor $T$, we use $\|M\|_F$ and $\|T\|_F$ to denote their Frobenius norm, which is equal to the $\ell_2$ norm of their vectorization.

For simplicity we restrict our attention to symmetric 4-th order tensors. For a vector $v\in \R^d$, we use $v^{\otimes 4}$ to denote a $d\times d\times d\times d$ tensor whose $(i,j,k,l)$-th entry is equal to $v_iv_jv_kv_l$. Suppose $T=\sum_{w} w^{\otimes 4},$ we define $T(v^{\otimes 4})$ as $\sum_{w} \inner{w}{v}^4$, $T(v^{\otimes 3}, I)$ as $\sum_{w} \inner{w}{v}^3 w$, and $T(v^{\otimes 2}, u, I) = \sum_w \inner{w}{v}^2 \inner{w}{u} w$.

For clarity, we always call a component in $T^*$ as ground truth component and call a component in our model $T$ simply as component.

\paragraph{Problem setup} We consider the problem of fitting a 4-th order tensor. The components of the ground truth tensor is arranged as columns of a matrix $U \in \R^{d\times r}$, and the tensor $T^*$ is defined as
\[T^* = \sum_{i=1}^r a_i (U[:,i]^{\otimes 4}),\]
where $a_1\geq a_2\geq \cdots \geq a_r\geq 0$ and $\sum_{i=1}^r a_i = 1$. For convenience in the analysis, we assume $a_i\geq \epsilon/\sqrt{d}$ for all $i\in [r].$ This is without loss of generality because the target accuracy is $\epsilon$ and we can safely ignore very small ground truth components with $a_i<\epsilon/\sqrt{d}.$
In this paper, we focus on the case where the components are orthogonal---that is, the columns $U[:,i]$'s are orthonormal. For simplicity we assume without loss of generality that $U[:,i] = e_i$ where $e_i$ is the $i$-th standard basis vector\footnote{This is without loss of generality because gradient flow (and our modifications) is invariant under rotation of the ground truth parameters.}. To reduce the number of parameters we also assume $r = d$, again this is without loss of generality because we can simply set $a_i = 0$ for $i > r$.

There can be many different ways to parametrize the tensor that we use to fit $T^*$. Following previous works~\citep{wang2020beyond,li2020learning}, we use an over-parameterized and two-homogeneous tensor 
\[T=\sum_{i=1}^m \frac{W[:,i]^{\otimes 4}}{\ns{W[:,i]}} .\]
Here $W\in \R^{d\times m}$ is a matrix with $m$ columns that corresponds to the components in $T$. It is overparametrized when $m > r$. 

Since the tensor $T$ only depends on the set of columns $W[:,i]$ instead of the orderings of the columns, for the most part of the paper we will instead write the tensor $T$ as
\begin{equation*}
    T=\sum_{w\in \mathrm{col}(W)} \frac{w^{\otimes 4}}{\ns{w}},
\end{equation*}

where $\colW$ is the set of all the column vectors in $W$. This allows us to discuss the dynamics of coordinates for a component $w$ without using the index for the component. In particular, $w_i$ always represents the $i$-th coordinate of the vector $w$. This representation is similar to the mean-field setup \citep{chizat2018global,mei2018mean} where one considers a distribution on $w$, however since we do not rely on analysis related to infinite-width limit we use the sum formulation instead. For the ease of presentation, we choose to restrict our setting to fourth-order
tensor decomposition, but our results can be easily generalized to tensor with order at least three.

\section{Tensor deflation process and tensor power method}
\label{sec:deflation}
In this section we will first discuss the basic tensor deflation process for orthogonal tensor decomposition. Then we show the connection between the tensor power method and gradient flow.

\paragraph{Tensor deflation} For orthogonal tensor decomposition, a popular approach is to first fit the largest ground truth component in the tensor, then subtract it out and recurse on the residual. The general process is given in Algorithm~\ref{algo:deflate}. In this process, there are multiple ways to find the best rank-1 approximation. For example, \cite{anandkumar2014tensor} uses tensor power method, which picks many random vectors $w$, and update them as $w = T^*(w^{\otimes 3},I)/\n{T^*(w^{\otimes 3},I)}$.

\begin{algorithm}[htbp]
\caption{Tensor Deflation Process}\label{algo:deflate}
\begin{algorithmic}
\STATE \textbf{Input:} Tensor $T^*$
\STATE \textbf{Output:} Components $W$ such that $T^* \approx \sum_{w\in \colW} w^{\otimes 4}/\|w\|^2$
\STATE{Initially let the residual $R$ be $T^*$.}
\WHILE{$\|R\|_F$ is large}
\STATE Find the best rank 1 approximation $w^{\otimes 4}/\|w\|^2$ for $R$.
\STATE Add $w$ as a new column in $W$, and let $R = R - w^{\otimes 4}/\|w\|^2$.
\ENDWHILE
\end{algorithmic}
\end{algorithm}

\paragraph{Tensor power method and gradient flow}

If we run tensor power method using a tensor $T^*$ that is equal to $\sum_{i=1}^d a_i e_i^{\otimes 4}$, then a component $w$ will converge to the direction of $e_i$ where $i$ is equal to $\arg\max_i a_i \bw_i^2$. If there is a tie (which happens with probability 0 for random $w$), then the point will be stuck at a saddle point.

Let's consider running gradient flow on $W$ with objective function $\frac{1}{2}\fns{T-T^*}$ as $T:=\sum_{w\in\colW}w^{\otimes 4}/\|w\|^2$. If $T$ does not change much, the residual $R := T^*-T$ is close to a constant. 
In this case the trajectory of one component $w$ is determined by the following differential equation:
\begin{equation}
  \label{eq:powerdynamics}
  \frac{\rd w}{\rd t} 
  = 4 R(\bw^{\otimes 2},w, I) - 2R(\bw^{\otimes 4})w.
\end{equation}
To understand how this process works,  we can take a look at $\frac{\rd w_i^2/\rd t}{w_i^2}$ (intuitively this corresponds to the growth rate for $w_i^2$). If $R \approx T^*$ then we have:
$$
    \frac{\rd w_i^2/\rd t}{w_i^2} \approx 8a_i\bw_i^2 - 4\sum_{j\in [d]}a_j\bw_j^4.
$$
From this formula it is clear that the coordinate with larger $a_i \bw_i^2$ has a faster growth rate, so eventually the process will converge to $e_i$ where $i$ is equal to $\arg\max_i a_i \bw_i^2$, same as the tensor power method. Because of their similarity later we refer to dynamics in Eqn. \eqref{eq:powerdynamics} as tensor power dynamics.

\section{Our algorithm}
\label{sec:algorithm}
Our algorithm is a modified version of gradient flow as described in Algorithm~\ref{algo:main}. First, we change the loss function to
\[
L(W) = \frac{1}{2}\fns{T-T^*}+\frac{\lambda}{2}\fns{W}.
\]

The additional small regularization $\frac{\lambda}{2}\fns{W}$ allows us to prove a {\em local stability} result that shows if there are components $w$ that are close to the ground truth components in direction, then they will not move too much (see Section~\ref{sec:induction_sketch}). 

Our algorithm runs in multiple epochs with increasing length. We use $\Wst$ to denote the weight matrix in epoch $s$ at time $t$. We use similar notation for tensor $\Tst.$ In each epoch we try to fit ground truth components with  $a_i \ge \beta^{(s)}.$ In general, the time it takes to fit one ground truth direction is inversely proportional to its magnitude $a_i$. The earlier epochs have shorter length so only large directions can be fitted, and later epochs are longer to fit small directions. 

At the middle of each epoch, we reinitialize all components that do not have a large norm. This serves several purposes: first we will show that all components that exceed the norm threshold will have good correlation with one of the ground truth components, therefore giving an initial condition to the local stability result; second, the reinitialization will reduce the dependencies between different epochs and allow us to analyze each epoch almost independently. These modifications do not change the dynamics significantly, however they allow us to do a rigorous analysis. 

\begin{algorithm}[htbp]
\caption{Modified Gradient Flow}\label{algo:main}
\begin{algorithmic}
\STATE \textbf{Input:} Number of components $m$, initialization scale $\delta_0$, re-initialization threshold $\delta_1$, increasing rate of epoch length $\gamma$, target accuracy $\epsilon$, regularization coefficient $\lambda$
\STATE \textbf{Output:} Tensor $T$ satisfying $\fn{T-T^*}\leq \epsilon.$
\STATE {Initialize $W^{(0,0)}$ as a $d\times m$ matrix with each column $w^{(0,0)}$ i.i.d. sampled from $\delta_0 \text{Unif}(\mathbb{S}^{d-1})$}.
\STATE {$\beta^{(0)}\leftarrow \fn{T^{(0,0)}-T^*}$; $s\leftarrow 0$}
\WHILE {$\fn{T^{(s,0)}-T^*}>\epsilon$}
\STATE {Phase 1: Starting from $W^{(s,0)}$, run gradient flow for time $t_1^{(s)}= O(\frac{d}{\beta^{(s)}\log(d)})$.}
\STATE {Reinitialize all components that have $\ell_2$ norm less than $\delta_1$ by sampling i.i.d. from $\delta_0 \text{Unif}(\mathbb{S}^{d-1})$.}
\STATE {Phase 2: Starting from $W^{(s,t_1^{(s)})}$, run gradient flow for $t_2^{(s)}-t_1^{(s)}= O(\frac{\log(1/\delta_1)+\log(1/\lambda)}{\beta^{(s)}})$ time}
\STATE {$W^{(s+1,0)}\leftarrow W^{(s,t_2^{(s)})};\ \beta^{(s+1)}\leftarrow \beta^{(s)}(1-\gamma)$; $s\leftarrow s+1$}
\ENDWHILE
\end{algorithmic}
\end{algorithm}

\section{Main theorem and proof sketch}
\label{sec:sketch}
In this section we discuss the ideas to prove the following main theorem\footnote{In the theorem statement, we have a parameter $\alpha$ that is not used in our algorithm but is very useful in the analysis (see for example Definition~\ref{def:sst}). Basically, $\alpha$ measures the closeness between a component and its corresponding ground truth direction (see more in Section~\ref{sec:induction_sketch}).}
  

\begin{restatable}{theorem}{maintheorem}\label{thm:main}
For any $\epsilon \ge \exp(-o(d / \log d)) $,
there exists $\gamma = \Theta(1)$, $m=\poly(d)$, $\lambda=\min\{O(\log d/d),O(\epsilon/d^{1/2})\})$, $\alpha=\min\{O(\lambda /d^{3/2}), O(\lambda^2), O(\epsilon^2/d^4)\}$, $\delta_1=O(\alpha^{3/2}/m^{1/2})$, $\delta_0=\Theta(\delta_1\alpha/\log^{1/2} (d))$ such that with probability $1-1/\poly(d)$ in the (re)-initializations, Algorithm~\ref{algo:main} terminates in $O(\log(d/\epsilon))$ epochs and returns a tensor $T$ such that \[\fn{T-T^*}\leq \epsilon.\]
\end{restatable}

Intuitively, epoch $s$ of Algorithm~\ref{algo:main} will try to discover all ground truth components with $a_i$ that is at least as large as $\beta^{(s)}$. The algorithm does this in two phases. In Phase 1, the small components $w$ will evolve according to tensor power dynamics. For each ground truth component with large enough $a_i$ that has not been fitted yet, we hope there will be at least one component in $W$ that becomes large and correlated with $e_i$. We call such ground truth components ``discovered''. Phase 1 ends with a check that reinitilizes all components with small norm. Phase 2 is relatively short, and in Phase 2 we guarantee that every ground truth component that has been discovered become ``fitted'', which means the residual $T-T^*$ becomes small in this direction.

However, there are still many difficulties in analyzing each of the steps. In particular, why would ground truth components that are fitted in previous epochs remain fitted? How to guarantee only components that are correlated with a ground truth component grow to a large norm? Why wouldn't the gradient flow in Phase 2 mess up with the initialization we require in Phase 1? We discuss the high level ideas to solve these issues. In particular, in Section~\ref{sec:induction_sketch} we first give an induction hypothesis that is preserved throughout the algorithm, which guarantees that every ground truth component that is fitted remains fitted. In Section~\ref{sec:phase1_sketch} we discuss the properties in Phase 1, and in Section~\ref{sec:phase2_sketch} we discuss the properties in Phase 2.

\subsection{Induction hypothesis and local stability}\label{sec:induction_sketch}

In order to formally define what it means for a ground truth component to be ``discovered'' or ``fitted'', we need some more definitions and notations.

\begin{definition}\label{def:sst}
Define $\Sst_i\subseteq [m]$ as the subset of components that satisfy the following conditions: the $k$-th component is in $\Sst_i$ if and only if there exists some time $(s',t')$ that is no later than $(s,t)$ and no earlier than the latest re-initialization of $W[:,k]$ such that
\[\n{W^{(s',t')}[:,k]}=\delta_1 \text{ and } [\overline{W^{(s',t')}[:,k]}_i]^2\geq 1-\alpha^2.\]
We say that ground truth component $i$ is {\em discovered} in epoch $s$ at time $t$, if  $\Sst_i$ is not empty.
\end{definition}

Intuitively, $\Sst_i$ is a subset of components in $W$ such that they have large enough norm and good correlation with the $i$-th ground truth component. Although such components may not have a large enough norm to fit $a_i$ yet, their norm will eventually grow. Therefore we say ground truth component $i$ is discovered when such components exist.

For convenience, we shorthand $\wst \in \{\Wst[:, j] | j\in \Sst_i\}$ by $\wst \in \Sst_i.$ 
Now we will discuss when a ground truth component is fitted, for that, let 
\[\hat{a}_i^{(s,t)} = \sum_{\wst\in S^{(s,t)}_i}\ns{\wst}.\]
Here $\hat{a}_i^{(s,t)}$ is the total squared norm for all the components in $\Sst_i$. We say a ground truth component is {\em fitted} if $a_i - \hat{a}_i^{(s,t)} \le 2\lambda$.

Note that one can partition the columns in $W$ using sets $\Sst_i$, giving $d$ groups and one extra group that contains everything else. We define the extra group as $S^{(s,t)}_\varnothing := [m]\setminus \bigcup_{k\in [d]}S^{(s,t)}_k$. 

For each of the non-empty $\Sst_i$, we can take the average of its component (weighted by $\ns{\wst}$):
\[\Est_{i,w} f(\wst):=\frac{1}{\hast_i}\sum_{\wst\in S^{(s,t)}_i}\ns{\wst} f(\wst).\]
If $S^{(s,t)}_i =  \varnothing,$ we define $\Est_{i,w} f(\wst)$ as zero.  
Now we are ready to state the induction hypothesis:

\begin{restatable}[Induction hypothesis]{prop}{proposition}
\label{prop:main}
  In the setting of Theorem~\ref{thm:main}, for any epoch $s$ and time $t$ and every $k \in [d]$, the following hold.
  \begin{enumerate}[(a)]
    \item \label{itm: Ist, individual}
      For any $\wst \in \Sst_k$, we have $\br{\bwst_k}^2 \ge 1 - \alpha$.
    \item \label{itm: Ist, average}
      If $\Sst_k$ is nonempty, $\Est_{k, w} \br{\bwst_k}^2 \ge 
      1 - \alpha^2 - 4sm\delta_1^2$. 
    \item \label{itm: Ist, residual}
      We always have $a_k - \hat{a}^{(s, t)}_k \ge \lambda/6 - s m\delta_1^2$;
      if $a_k \geq \frac{\beta^{(s)}}{1-\gamma}$, we further know 
      $a_k - \hat{a}_k^{(s,t)} \le \lambda+ s m\delta_1^2$. 
    \item If $\wst \in \Sst_\varnothing$, then $\|\wst\| \le \delta_1$.
  \end{enumerate}
\end{restatable}

We choose $\delta_1^2$ small enough so that $sm\delta_1^2$ is negligible compared with $\alpha^2$ and $\lambda.$
Note that if Proposition~\ref{prop:main} is maintained throughout the algorithm, all the large components will be fitted, which directly implies Theorem~\ref{thm:main}. Detailed proof is deferred to Appendix~\ref{sec:proof_main_theorem}.

Condition (c) shows that for a ground truth component $k$ with large enough $a_k$, it will always be fitted after the corresponding epoch
(recall from Theorem~\ref{thm:main} that $\lambda = 
O(\eps/\sqrt{d})$).
Condition (d) shows that components that did not discover any ground truth components will always have small norm (hence negligible in most parts of the analysis). Conditions (a)(b) show that as long as a ground truth component $k$ has been discovered, all components that are in $S^{(s,t)}_k$ will have good correlation, while the {\em average} of all such components will have even better correlation. The separation between individual correlation and average correlation is important in the proof. With only individual bound, we cannot maintain the correlation no matter how small $\alpha$ is. Here is an example below:

\begin{restatable}{claim}{example}\label{clm:example}
Suppose $T^* = e_k^{\otimes 4}$ and $T=v^{\otimes 4}/\ns{v} + w^{\otimes 4}/\ns{w}$ with $\ns{w}+\ns{v}\in[2/3,1].$ Suppose $\bv_k^2 = 1-\alpha$ and $\bv_k = \bw_k, \bv_{-k}=-\bw_{-k}.$ Assuming $\ns{v}\leq c_1$ and $\alpha\leq c_2$ for small enough constants $c_1,c_2,$ we have 
$\frac{\rd}{\rd t}\bv_k^2<0.$
\end{restatable}

In the above example, both $\bar{v}$ and $\bar{w}$ are close to $e_k$ but they are opposite in other directions ($\bv_{-k}=\bw_{-k}$). The norm of $v$ is very small compared with that of $w$. Intuitively, we can increase $v_{-k}$ so that the average of $v$ and $w$ is more aligned with $e_k$. See the rigorous analysis in Appendix~\ref{sec: induction, counterexample}. 

The induction hypothesis will be carefully maintained throughout the analysis. The following lemma guarantees that in the gradient flow steps the individual and average correlation will be maintained.

\begin{lemma}
  \label{lem:individual_average} In the setting of Theorem~\ref{thm:main},
  suppose Proposition~\ref{prop:main} holds in epoch $s$ at time $t$, we have 
  \begin{align*}
    \frac{\rd}{\rd t} [\bwst]^2
    &\ge 8 \left( a_k - \hat{a}^{(s, t)}_k \right) 
      \left( 1 - [\bwst_k]^2 \right) 
        - O\left( \alpha^{1.5} \right), \\
    \frac{\rd}{\rd t} \Est_{k, w}[\bwst_k]^2
    &\ge 8 \left( a_k - \hat{a}^{(s, t)}_k \right) 
     \left( 1-\Est_{k,w}[\bwst_k]^2 \right) - O(\alpha^3).
  \end{align*}
  In particular, when $a_k-\hat{a}^{(s, t)}_k\geq \Omega(\lambda) = \Omega(\sqrt{\alpha}),$ we have $\frac{\rd}{\rd t} [\bwst_k]^2 >0$ when $[\bwst_k]^2=1-\alpha$ and  $\frac{\rd}{\rd t} \Est_{k, w}[\bwst_k]^2 >0$ when $\Est_{k, w}[\bwst_k]^2=1-\alpha^2.$
\end{lemma}

The detailed proof for the local stability result can be found in Appendix~\ref{sec: appendix, induction hypothesis}.
Of course, to fully prove the induction hypothesis one needs to talk about what happens when a component enters $S^{(s,t)}_i$, and what happens at the reinitialization steps. We discuss these details in later subsections.

\vspace{-1mm}
\subsection{Analysis of Phase 1}\label{sec:phase1_sketch}

In Phase 1 our main goal is to discover all the components that are large enough. We also need to maintain Proposition~\ref{prop:main}. Formally we prove the following:

\begin{restatable}[Main Lemma for Phase 1]{lemma}{lemphaseone}\label{lem:phase1} 
In the setting of Theorem~\ref{thm:main}, suppose Proposition~\ref{prop:main} holds at $(s,0).$ For $t_1^{(s)}:=t_1^{(s)\prime}+t_1^{(s)\dbprime}+t_1^{(s)\triprime}$ with $t_1^{(s)\prime} = \Theta(d/(\beta^{(s)} \log d))$, $t_1^{(s)\dbprime}=\Theta(d/(\beta^{(s)} \log^3 d))$, $t_1^{(s)\triprime}=\Theta(\log(d/\alpha)/\beta^{(s)})$, with probability $1-1/\poly(d)$ we have
\begin{enumerate}
    \item Proposition \ref{prop:main} holds at $(s,t)$ for any $0\leq t <  t_1^{(s)}$, and also for $t = t_1^{(s)}$ after reinitialization.
    \item If $a_k\geq \beta^{(s)}$ and $S^{(s,0)}_k=\varnothing$, we have $S_k^{(s,t_1^{(s)})}\neq \varnothing$ and $\hat{a}_k^{(s,t_1^{(s)})}\geq \delta_1^2.$
    \item If $S^{(s,0)}_k=\varnothing$ and $S^{(s, t_1^{(s)})}_k\neq \varnothing,$ we have $a_k\geq C\beta^{(s)}$ for universal constant $0<C<1$. 
\end{enumerate}
\end{restatable}

Property 2 shows that large enough ground truth components are always discovered, while Property 3 guarantees that no small ground truth components can be discovered. Our proof relies on initial components being ``lucky'' and having higher than usual correlation with one of the large ground truth components. To make this clear we separate components into different sets (here we use $v$ to denote a component in $W$):

\begin{definition}[Partition of (re-)initialized components]\label{def-phase1-partition}
    For each direction $i\in[d]$, define the set of good components $\Ss_{i,good}$ and the set of potential components $\Ss_{i,pot}$ as follow, where $\Gamma^{(s)}_i:=1/(8a_i t_1^{(s)\prime})$ if $S_i^{(s,0)}=\varnothing$, and $\Gamma^{(s)}_i := 1/(8\lambda t_1^{(s)\prime})$ otherwise. Here $\rho^{(s)}_i:=c_\rho \Gamma^{(s)}_i$ and $c_\rho$ is a small enough absolute constant. 
    \begin{align*}
        \Ss_{i,good} &:= \{k \mid [\bv^{(s,0)}_i]^2\ge \Gamma^{(s)}_i+\rho^{(s)}_i,\  [\bv^{(s,0)}_j]^2 \le \Gamma^{(s)}_j-\rho^{(s)}_j,\forall j\ne i \text{ and } v^{(s,0)} = W^{(s,0)}[:,k]\},\\
        \Ss_{i,pot}  &:= \{k\mid [\bv^{(s,0)}_i]^2\ge \Gamma^{(s)}_i-\rho^{(s)}_i\text{ and } v^{(s,0)} = W^{(s,0)}[:,k]\}.
    \end{align*}
    Let $\Ss_{good}:= \cup_i \Ss_{i,good}$ and $\Ss_{pot}:= \cup_i \Ss_{i,pot}$.
    We also define the set of bad components $\Ss_{bad}$.
    \begin{align*}
        \Ss_{bad}    &:= \{k\mid \exists i\ne j \text{ s.t. } [\bv^{(s,0)}_i]^2\ge\Gamma^{(s)}_i-\rho^{(s)}_i,\  [\bv^{(s,0)}_j]^2\ge\Gamma^{(s)}_j-\rho^{(s)}_j\text{ and } v^{(s,0)} = W^{(s,0)}[:,k]\}.
    \end{align*}
\end{definition}

For convenience, we shorthand $\vst \in \{\Wst[:, j] | j\in S_{i,good}\}$ by $\vst \in S_{i,good}$ (same for $S_{i,pot}$ and $S_{bad}$).
Intuitively, the good components will grow very quickly and eventually pass the norm threshold. Since both good and potential components only have one large coordinate, they will become correlated with that ground truth component when their norm is large. The bad components are correlated with two ground truth components so they can potentially have a large norm while not having a very good correlation with either one of them. In the proof we will guarantee with probability at least $1-1/\poly(d)$ that good components exists for all large enough ground truth components and there are no bad components. The following lemma characterizes the trajectories of different type of components:

\begin{restatable}{lemma}{lemphaseonesummarytrajectory}\label{lem-phase1-summary-trajectory}
    In the setting of Lemma~\ref{lem:phase1}, for every $i \in[d]$
    \begin{enumerate}
        \item (\emph Only good/potential components can become large) If $\vst \not\in \Ss_{pot}$, $\n{\vst}=O(\delta_0)$ and $[\bvst_i]^2=O(\log(d)/d)$ for all $i\in[d]$ and $t\le t_1^{(s)}$.
    
        \item (\emph Good components discover ground truth components) If $\Ss_{i,good}\neq \varnothing$, there exists $v^{(s,t_1^{(s)})}$ such that $\n{v^{(s,t_1^{(s)}})}\ge \delta_1$ and $S_i^{(s,t_1^{(s)})}\neq\varnothing$. 
        
        
        \item (\emph Large components are correlated with ground truth components) If $\n{\vst}\geq \delta_1$ for some $t\leq t_1^{(s)}$, there exists $i\in [d]$ such that $\vst\in \Sst_i$. 
    \end{enumerate}
\end{restatable}

The proof of Lemma~\ref{lem-phase1-summary-trajectory} is difficult as one cannot guarantee that all the ground truth components that we are hoping to fit in the epoch will be fitted simultaneously. 
However we are able to show that $T-T^*$ remains near-orthogonal and control the effect of changing $T-T^*$ within this epoch. The details are in Appendix~\ref{sec:proof_init_phase1}.

\vspace{-1mm}
\subsection{Analysis of Phase 2}\label{sec:phase2_sketch}

In Phase 2 we will show that every ground truth component that's discovered in Phase 1 will become fitted, and the reinitialized components will preserve the desired initialization conditions.
\begin{restatable}[Main Lemma for Phase 2]{lemma}{phasetwomain}\label{lem:phase2}
In the setting of Theorem~\ref{thm:main}, suppose Proposition~\ref{prop:main} holds at $(s,t_1^{(s)}),$ we have for $t_2^{(s)}-t_1^{(s)}:=O(\frac{\log(1/\delta_1)+\log(1/\lambda)}{\beta^{(s)}})$
\begin{enumerate}
    \item Proposition~\ref{prop:main} holds at $(s,t)$ for any $t_1^{(s)}\leq t\leq t_2^{(s)}.$
    \item If $S_k^{(s, t_1^{(s)})}\neq \varnothing,$ we have $a_k-\hat{a}^{(s,t_2^{(s)})}_k\leq 2\lambda.$
    \item For any component $v$ that was reinitialized at $t_1^{(s)}$, we have $\ns{v^{(s,t_2^{(s)}}} = \Theta(\delta_0^2)$ and $\br{\bv_i^{(s,t_2^{(s)})}}^2 = \br{\bv_i^{(s,t_1^{(s)})}}^2 \pm o\pr{\frac{\log d }{d}}$ for every $i\in [d].$
\end{enumerate}
\end{restatable}

The main idea is that as long as a direction has been discovered, the norm of the corresponding components will increase very fast. The rate of that is characterized by the following lemma.

\begin{lemma}[informal]In the setting of Theorem~\ref{lem:phase2},
  for any $t_1^{(s)}\leq t\leq t_2^{(s)},$
  \[
    \frac{\rd}{\rd t} \hat{a}^{(s,t)}_k
    \ge \pr{2 (a_k-\hat{a}^{(s,t)}_k) - \lambda - O\left(\alpha^2 \right)}\hat{a}^{(s,t)}_k.
  \]
  In particular, after $O(\frac{\log(1/\delta_1)+\log(1/\lambda)}{a_k})$ time, we have $a_k-\hat{a}^{(s,t)}_k\leq \lambda.$
\end{lemma}

By the choice of $\delta_1$ and $\lambda$, the length of Phase 2 is much smaller than the amount of time needed for the reinitialized components to move far, allowing us to prove the third property in Lemma~\ref{lem:phase2}. Detailed analysis is deferred to Appendix~\ref{sec:proof_phase2}.

\section{Conclusion}\label{sec:conclude}

In this paper we analyzed the dynamics of gradient flow for over-parametrized orthogonal tensor decomposition. With very mild modification to the algorithm (a small regularizer and some re-initializations), we showed that the trajectory is similar to a tensor deflation process and the greedy low-rank procedure in~\citet{li2020towards}. These modifications allowed us to prove strong guarantees for orthogonal tensors of any rank, while not changing the empirical behavior of the algorithm. We believe such techniques would be useful in later analysis for the implicit bias of tensor problems.

A major limitation of our work is that it only applies to orthogonal tensors. Going beyond this would require significantly new ideas---we observed that for general tensors, overparametrized gradient flow may have a very different behavior compared to the greedy low-rank procedure, as it is possible for two large component in the same direction to split into two different directions (see more details in Appendix~\ref{sec:experiment}). We leave that as an interesting open problem.

\section*{Acknowledgements}
Rong Ge, Xiang Wang and Mo Zhou are supported in part by NSF Award CCF-1704656, CCF-1845171 (CAREER), CCF-1934964 (Tripods), a Sloan Research Fellowship, and a Google Faculty Research Award.

\bibliography{ref}
\bibliographystyle{apalike}

\newpage
\appendix

\section*{Overview of Supplementary Materials}

In the supplementary material we will give detailed proof for Theorem~\ref{thm:main}. We will first highlight a few technical ideas that goes into the proof, and then give details for each part of the proof.

\paragraph{Continuity Argument}

Continuity argument is the main tool we use to 
prove Proposition~\ref{prop:main}.
Intuitively, the continuity argument says that if whenever a property is about
to be violated, there exists a positive speed that pulls it back, then that 
property will never be violated. In some sense, this is the continuous version
of the mathematical induction or, equivalently, the minimal counterexample 
method. See Section 1.3 of \cite{tao_nonlinear_2006} for a short discussion on this 
method. 

However, since our algorithm is not just gradient flow, and in particular involves reinitialization steps that are not continuous, we need to generalize continuity argument to handle impulses. We give detailed lemmas in Section~\ref{sec: continuity argument} as the continuity argument is mostly used to prove Proposition~\ref{prop:main}.

\paragraph{Approximating residual} In many parts of the proof, we approximate the residual $T^* - T$ as:
\[
T^* - T = \sum_{i=1}^d \tilde{a}_i e_i^{\otimes 4} + \Delta,
\]
where $\tilde{a}_i = a_i -\hat{a}_i.$
That is, we think of $T^* - T$ as an orthogonal tensor with some perturbations. The norm of the perturbation $\|\Delta\|_F$ is going to be bounded by $O(\alpha + m\delta_1^2)$, which is sufficient in several parts of the proof that only requires crude estimates. However, in several key steps of our proof (including conditions (a) and (b) of Proposition~\ref{prop:main} and the analysis of the first phase), it is important to use extra properties of $\Delta$. In particular we will expand $\Delta$ to show that for a basis vector $e_i$ we always have $\Delta(e_i^{\otimes 4}) = o(\alpha)$, which gives us tighter bounds when we need them.

\paragraph{Radial and tangent movement} Throughout the proof, we often need to track the movement of a particular component $w$ (a column in $W$). It is beneficial to separate the movement of $w$ into radial and tangent movement, where radial movement is defined as $\inner{\frac{dw}{dt}}{w}$ and tangent movement is defined as $P_{w^\perp} \frac{dw}{dt}$ (where $P_{w^\perp}$ is the projection to the orthogonal subspace of $w$). Intuitively, the radial movement controls the norm of the component $w$, and the tangent movement controls the direction of $w$. When the component $w$ has small norm, it will not significantly change the residual $T^*-T$, therefore we mostly focus on the tangent movement; on the other hand when norm of $w$ becomes large in our proof we show that it must already be correlated with one of the ground truth components, which allow us to better control its norm growth.

\paragraph{Overall structure of the proof} The entire proof is a large induction/continuity argument which maintains Proposition~\ref{prop:main} as well as properties of the two phases (summarized later in Assumption~\ref{assumption: induction, oracle}). In each part of the proof, we show that if we assume these conditions hold for the previous time, then they will continue to hold during the phase/after reinitialization. 

In Section~\ref{sec: appendix, induction hypothesis} we prove Proposition~\ref{prop:main} assuming Assumption~\ref{assumption: induction, oracle} holds before. In Section~\ref{sec:proof_phase1} we prove guarantees of Phase 1 and reinitialization assuming Proposition~\ref{prop:main}. In Seciton~\ref{sec:proof_phase2} we prove guarantees for Phase 2 assuming Proposition~\ref{prop:main}. Finally in Section~\ref{sec:proof_main_theorem} we give the proof of the main theorem.

\paragraph{Experiments} Finally in Section~\ref{sec:exp_detail} we give details about experiments that illustrate the deflation process, and show why such a process may not happen for non-orthgonal tensors.

\section{Proofs for Proposition~\ref{prop:main}}
\label{sec: appendix, induction hypothesis}

The goal of this section is to prove Proposition~\ref{prop:main} under 
Assumption~\ref{assumption: induction, oracle}. We also prove 
Claim~\ref{clm:example} in Section~\ref{sec: induction, counterexample}.

\paragraph{Notations} Recall we defined \[\Est_{i,w} f(\wst):=\frac{1}{\hast_i}\sum_{\wst\in S^{(s,t)}_i}\ns{\wst} f(\wst).\]
We will use this notation extensively in this section. For simplicity, we shall drop the superscript of epoch $s$. Further, we sometimes consider expectation with two variables $v$ and $w$:
\[\Est_{i,v,w} f(\wst):=\frac{1}{\left[\hast_i\right]^2}\sum_{\vst,\wst\in S^{(s,t)}_i}\ns{\vst}\ns{\wst} f(\wst,\vst).\]
We will also use $z_t$ to denote $\zt := \inner{\bvt}{\bwt}$ and $\tat_k := a_k - \hatt_k$. Note that $v$ and $w$ in this section (and later in the proof) just serve as arbitrary components in columns of $W$.

\begin{assumption}
  \label{assumption: induction, oracle}
  Throughout this section, we assume the following.
  \begin{enumerate}[(a)]
      
    \item For any $k \in [d]$, in phase 1, when $\|\vt\|$ enters $\St_k$, 
      that is, $\|\vt\| = \delta_1$, we have $[\bvt_k]^2 \ge 1 - \alpha^2$ if 
      $\hatt_k < \alpha$ and $[\bvt_k]^2 \ge 1 - \alpha$ if $\hatt_k \ge \alpha$. 
    
    \item There exists a small constant $c > 0$ s.t.~for any $k \in [d]$ with 
      $a_k < c\beta^{(s)}$, in phase 1, no components will enter $\St_k$.
    
    \item For any $k \in [d]$, in phase 2, no components will enter $\St_k$. 
    
    \item For the parameters, we assume $m\delta_1^2 \le \alpha^3$ and 
      $\Omega\left( \sqrt{\alpha} \right) \le \lambda 
      \le O\left( \min_{s} \beta^{(s)} \right) = O(\eps / \sqrt{d})$. 
  \end{enumerate}
\end{assumption}
\begin{remark}
  As we mentioned, the entire proof is an induction and we only need the assumption up to the point that we are analyzing. The assumption will be proved later 
  in Appendix~\ref{sec:proof_init_phase1} and \ref{sec:proof_phase2} to finish the induction/continuity argument. 
 The reason 
  we state this assumption here, and state it as an assumption, is to make
  the dependencies more transparent. 
\end{remark}

\begin{remark}[Remark on the choice of $\lambda$]
  The lower bound $\lambda = \Omega(\sqrt{\alpha})$ comes from 
  Lemma~\ref{lemma: individual bound}. For the upper bound, first
  note that when $\lambda$ is larger than $a_k$, actually the norm of components
  in $\St_k$ can decrease (cf.~Lemma \ref{lemma: d |v|2}). Hence, we require
  $\lambda < c \min_s \beta^{(s)} / 10$ where $c$ is the constant in (c).
  This makes sure in phase 2 the growth rate of $\hatt_k$ is not too small. 
\end{remark}

\proposition*

Before we move on to the proof, we collect some further remarks on 
Proposition~\ref{prop:main} and the proof overview here. 

\begin{remark}[Remark on the epoch correction term]
  Note that conditions (\ref{itm: Ist, average}) and (\ref{itm: Ist, residual}) 
  have an additional term with form $O(s m\delta_1^2)$. 
  This is because these average bounds may deteriorate a little when the content 
  of $\St_k$ changes, which will happen when new components enter $\St_k$ or the 
  reinitialization throw some components out of $\St_k$. 
  The norm of the components involved in these fluctuations is upper bounded by 
  $\delta_1$ and the number by $m$. Thus the $O(m\delta_1^2)$ factor. 
  The factor $s$ accounts for the accumulation across epochs. We need
  this to guarantee at the beginning of each epoch, the conditions hold
  with some slackness (cf.~Lemma~\ref{lemma: continuity argument with impulses}). 
  Though this issue can be fixed by a slightly sharper estimations for the 
  ending state of each epoch, adding one epoch correction term is simpler and,
  since we only have $\log (d/\epsilon)$ 
  epochs, it does not change the bounds too much and, in fact, we can always 
  absorb them into the coefficients of $\lambda$ and $\alpha^2$, respectively. 
\end{remark}

\begin{remark}[Remark on condition~(\ref{itm: Ist, individual})]
  Note that Assumption~\ref{assumption: induction, oracle} makes sure that when 
  a component enters $\St_k$, we always have $[\bvt_k]^2 \ge 1 - \alpha$. 
  Hence, essentially this condition says that it will remain basis-like.
  Following the spirit of the continuity argument, to maintain this condition,
  it suffices to prove Lemma~\ref{lemma: individual bound}, the proof of 
  which is deferred to Section~\ref{sec:individual}.
  Also note that by Assumption~\ref{assumption: induction, oracle} and the 
  definition of $\Sst_k$, neither the entrance of new components nor the 
  reinitialization will break this condition. 
\end{remark}

\begin{restatable}{lemma}{lemmaindividualbound}
  \label{lemma: individual bound}
  Suppose that at time $t$, Proposition~\ref{prop:main} is true. Assuming $\delta_1^2 = O(\alpha^{1.5}/m),$ then for any $\vt \in \St_k$, we have
  \[
    \frac{\rd}{\rd t} [\bvt]^2
    \ge 8 \tat \left( 1 - [\bvt_k]^2 \right) [\bvt_k]^4
      - O\left( \alpha^{1.5}\right),
  \]
  In particular, if $\lambda = \Omega\left(\sqrt{\alpha}\right)$, then 
  $\frac{\rd}{\rd t} [\bvt]^2 > 0$ whenevner $[\bvt_k]^2 = 1 - \alpha$.
\end{restatable}

\begin{remark}[Remark on condition~(\ref{itm: Ist, average})]
  The proof idea of condition~(\ref{itm: Ist, average}) is similar to 
  condition~(\ref{itm: Ist, individual}) and we prove Lemma~\ref{lemma: average 
  bound} in Section~\ref{sec:average}. In Section~\ref{sec:average}, we also 
  handle the impulses caused by the entrance of new components and the 
  reinitialization. 
\end{remark}

\begin{restatable}{lemma}{lemmaaveragebound}
  \label{lemma: average bound}
  Suppose that at time $t$, Proposition~\ref{prop:main} is true and 
  $\St_k \ne \varnothing$. Assuming $\delta_1^2 = O(\alpha^3/m)$, we have
  \[  
    \frac{\rd}{\rd t} \Et_{k, v}[\bvt_k]^2
    \geq 8\tat_k (1-\Et_{k,v}[\bvt_k]^2) - O(\alpha^3).
  \]
  In particular, if $\lambda = \Omega(\alpha)$, then 
  $\frac{\rd}{\rd t} \Et_{k, v}[\bvt_k]^2 > 0$ when 
  $\Et_{k, v}[\bvt_k]^2 < 1 - \alpha^2/2$.   
\end{restatable}

\begin{remark}[Remark on condition~(\ref{itm: Ist, residual})]
  This condition says that the residual along direction $k$ is always
  $\Omega(\lambda)$. This guarantees the existence of a small attraction 
  region around $e_k$, which will keep basis-like components basis-like. 
  We rely on the regularizer to maintain this condition. 
  The second part of condition~(\ref{itm: Ist, residual}) means fitted directions 
  will remain fitted. We prove Lemma~\ref{lemma: d hattk, upper and lower bound} 
  and handle the impulses in Section~\ref{sec:residual}. 
\end{remark}

\begin{lemma}[Lemma~\ref{lemma: d hattk, upper bound} and Lemma~\ref{lemma: 
  d hattk, lower bound}]
  \label{lemma: d hattk, upper and lower bound}
  Suppose that at time $t$, Proposition~\ref{prop:main} is true. and no impulses 
  happen at time $t$. Then at time $t$, we have
  \[
    \frac{1}{\hatt_k} \frac{\rd}{\rd t} \hatt_k
    = 2 \tat_k - \lambda \pm O\left(\alpha^2 \right).
  \]
  In particular, $\frac{\rd}{\rd t} \hatt_k$ is negative (resp.~positive)
  when $\hatt_k > a_k - \lambda / 6$ (resp.~$\hatt_k < a_k - \lambda$).
\end{lemma}

\subsection{Continuity argument}\label{sec: continuity argument} 

{
\newcommand{\ps}[1]{^{(#1)}}
\newcommand{\It}{\mathbf{I}}
\newcommand{\indi}[1]{\mathbbm{1}_{#1}}


We mostly use the following version of continuity argument, which is adapted from Proposition~1.21 of 
\cite{tao_nonlinear_2006}. 

\begin{lemma}
  \label{lemma: continuity argument}
  Let $\It\ps{t}$ be a statement about the structure of some object. 
  $\It\ps{t}$ is true for all $t \ge 0$ as long as the following hold.
  \begin{enumerate}[(a)]
    \item $\It\ps{0}$ is true.
    \item $\It$ is closed in the sense that for any sequence $t_n \to t$, if
      $\It\ps{t_n}$ is true for all $n$, then $\It\ps{t}$ is also true.
    \item If $\It\ps{t}$ is true, then there exists some $\delta > 0$ s.t. 
      $\It\ps{s}$ is true for $s \in [t, t+\delta)$. 
  \end{enumerate}
  In particular, if $\It\ps{t}$ has form $\bigwedge_{i=1}^N \bigvee_{j=1}^N 
  p\ps{t}_{i, j} \le q_{i, j}$. Then, we can replace (b) and (c) by the
  following.
  \begin{enumerate}
    \item[(b')] $p\ps{t}_{i, j}$ is $C^1$ for all $i,j$. 
    \item[(c')] Suppose at time $t$, $\It\ps{t}$ is true but some clause
    $\bigvee_{j=1}^N p\ps{t}_{i,j} \le q_{i,j}$ is tight, in the sense that
    $p\ps{t}_{i,j} \ge q_{i,j}$ for all $j$ with at least one equality. Then
    there exists some $k$ s.t. $p_{i,k}\ps{t} = q_{i,k}$ and
    $\dot{p}\ps{t}_{i, k} < 0$. 
  \end{enumerate}
\end{lemma}
\begin{proof}
  Define $t' := \sup\{t\ge 0 \;:\; \It\ps{t} \text{ is true}\}$. Since 
  $\It\ps{0}$ is true, $t' \ge 0$. Assume, to obtain a contradiction, that 
  $t' < \infty$. Since $\It$ is closed, $\It\ps{t'}$ is true, whence there 
  exists a small $\delta > 0$ s.t. $\It\ps{t}$ is true in $[t', t'+\delta)$. 
  Contradiction. 
  
  For the second set of conditions, first note that the continuity of 
  $p_{i,j}\ps{t}$ and the non-strict inequalities imply that $\It$ is closed. 
  Now we show that (b') and (c') imply (c). If none of the clause is
  tight at time $t$, by the continuity of $p_{i,j}\ps{t}$, $\It$ holds in a small
  neighborhood of $t$. If some constraint is tight, by (c') and the $C^1$
  condition, we have $p\ps{t}_{i, k} < q_{i, k}$ in a right small neighborhood
  of $t$. 
\end{proof}

\begin{remark}
  Despite the name ``continuity argument'', it is possible to generalize it
  to certain classes of discontinuous functions. In particular, we consider
  impulsive differential equations here, that is, for almost every $t$, 
  $p\ps{t}$ behaves like a usual differential equation, but at some 
  $t_i$, it will jump from $p\ps{t_i-}$ to $p\ps{t_i} = p\ps{t_i-} + \delta_i$. 
  See, for example, \cite{lakshmikantham_theory_1989} for a systematic treatment 
  on this topic. Suppose that we still want to maintain the property 
  $p\ps{t} \le 0$. If the total amount of impulses is small and we have some 
  cushion in the sense that $\dot{p}\ps{t} < 0$ whenever 
  $p\ps{t} \in [-\eps, 0]$ , then we can still hope $p\ps{t} \le 0$ to hold for 
  all $t$, since, intuitively, only the jumps can lead $p\ps{t}$ into 
  $[-\eps, 0]$, and the normal $\dot{p}\ps{t}$ will try to take it back
  to $(-\infty, -\eps)$. As long as the amount of impulses is smaller than
  the size $\eps$ of the cushion, then the impulses will never break things. We
  formalize this idea in the next lemma. 
\end{remark}

\begin{lemma}[Continuity argument with impulses]
  \label{lemma: continuity argument with impulses}
  Let $0 < t_1 < \cdots < t_N < \infty$ be the moments at which the
  impulse happens and $\delta_1, \dots, \delta_N \in \R$ the size of the
  impulses at each $t_i$. Let $p: [0, \infty) \to \R$ be a function
  that is $C^1$ on $[0, t_1)$, every $(t_i, t_{i+1})$ and $(t_N, \infty)$, and 
  $p\ps{t_i} = p\ps{t_i-} + \delta_i$. Write $\Delta = \sum_{i=1}^N 
  \max\{0, \delta_i\}$. If (a) $p\ps{0} \le -\Delta$ and (b) for every 
  $t \notin \{t_i\}_{i=1}^N$ with $p\ps{t} \in [-\Delta, 0]$, we have 
  $\dot{p}\ps{t} < 0$, then $p\ps{t} \le 0$ always holds. 
\end{lemma}
\begin{remark}
  Note that if there is no impulses, then $p\ps{t}$ is a usual $C^1$ function and
  we recover conditions (b') and (c') of Lemma~\ref{lemma: continuity argument}.
  Also, though the statement here only concerns one $a_t$, one can incorporate
  it into Lemma~\ref{lemma: continuity argument} by replacing (b') and (c')
  with the hypotheses of this lemma and modify (a) to be $p\ps{0}_{i, j} \le 
  p_{i, j} - \Delta_{i, j}$. 
\end{remark}
\begin{proof}
  We claim that $p\ps{t} \le -\Delta + \sum_{i=1}^N \indi{t \le t_k} 
  \max\{0, \delta_i\} =: q\ps{t}$. 
  Define $t' = \sup\{t \ge 0 \;:\; p\ps{t} \le q\ps{t}\}$. Since $p\ps{t} \le 
  -\Delta$ and $t_1 > 0$, $t' \ge 0$. Assume, to obtain a contradiction, that $t'
  < \infty$ and consider $p\ps{t'}$. If $t' = t_k$ for some $k$, then, by the 
  definition of $t'$, $p\ps{t'-} \le -\Delta + \sum_{i=1}^{k-1}
  \max\{0, \delta_i\}$, whence, $p\ps{t'} = p\ps{t'-} + \delta_k \le -\Delta 
  + \sum_{i=1}^{k} \max\{0, \delta_i\}$.
  Contradiction. If $t' \notin \{t_i\}_{i=1}^N$, then by the continuity of 
  $p$, we have $p\ps{t'} = q\ps{t'}$. Then, since $\dot{p}\ps{t'} < 0$ and $p$ 
  is $C^1$, we have $p\ps{t} < p\ps{t'} = q\ps{t'} = q\ps{t}$ in $[t', t'+\tau]$
  for some small $\tau > 0$, which contradicts the maximality of $t'$. 
  Thus, $p\ps{t} \le 0$ holds for all $t \ge 0$. 
\end{proof}

}

\subsection{Preliminaries}
\label{sec: induction, preliminaries}

The next two lemmas give formulas for the norm growth rate and tangent speed
of each component. 

\begin{lemma}[Norm growth rate]
  \label{lemma: d |v|2}
  For any $\vt$, we have
  \begin{equation*}
    \frac{1}{2 \ns{\vt}} \frac{\rd}{\rd t} \ns{\vt}
    = \sum_{i=1}^d a_i [\bvt_i]^4
      - \sum_{i=1}^d \hatt_i \Et_{i, w} \left\{ [\zt]^4 \right\}
      - \Tt_\varnothing \left( [\bvt]^{\otimes 4} \right)
      - \frac{\lambda}{2}. 
  \end{equation*}
\end{lemma}
\begin{proof}
  Due to the $2$-homogeneity, we have\footnote{In the mean-field terminologies,
  the RHS is just the first variation (or functional derivative) of the 
  loss at $\bvt$.}
  \begin{align*}
    \frac{1}{2 \ns{\vt}} \frac{\rd}{\rd t} \ns{\vt}
    &=  \left( T^* - \Tt \right) \left( [\bvt]^{\otimes 4} \right)
      - \frac{\lambda}{2}. 
  \end{align*}
  The ground truth terms can be rewritten as
  \[
    T^*\left( [\bvt]^{\otimes 4} \right) 
    = \sum_{i=1}^d a_i [\bvt_i]^4.
  \]
  Decompose the $\Tt$ term accordingly and we get
  \[
    \Tt\left( [\bvt]^{\otimes 4} \right) 
    = \sum_{i=1}^d \hatt \Et_{i, w} \left\{ [\zt]^4 \right\}
      + \Tt_\varnothing \left( [\bvt]^{\otimes 4} \right).
  \]
\end{proof}

\begin{lemma}[Tangent speed]
  \label{lemma: d vtk2}
  Suppose that at time $t$, Proposition~\ref{prop:main} is true.
  Then at time $t$, for any $\vt \in \Wt$ and any $k \in [d]$, we have
  \begin{align*}
    \frac{\rd}{\rd t} [\bvt]^2
    = G_1 - G_2 - G_3 \pm O(m\delta_1^2),
  \end{align*}
  where
  \begin{align*}
    G_1 
    &:= 8 a_k \left( 1 - [\bvt_k]^2 \right) [\bvt_k]^4
      - 8 \hatt_k \left( 1 - [\bvt_k]^2 \right) 
        \Et_{k, w}\left\{ [\zt]^4 \right\} \\
      &\qquad + 8 \hatt_k \Et_{k, w}\left\{ [\zt]^3 
        \inner{\bw_{-k}}{\bv_{-k}} \right\}, \\
    G_2
    &= 8 \sum_{i \ne k} \hatt_i \Et_{i, w} \left\{ 
        [\zt]^3 \vt_k \wt_k
      \right\}, \\
    G_3
    &= 8 [\bvt_k]^2 \sum_{i \ne k} \left(
        a_i [\bvt_i]^4 
        - \hatt_i \Et_{i, w} \left\{ [\zt]^4 \right\}
      \right).
  \end{align*}
\end{lemma}
\begin{remark}
  Intuitively, $G_1$ captures the local dynamics around $e_k$ and 
  $G_2$ characterize the cross interaction between different ground truth 
  directions. 
\end{remark}
\begin{proof}
  Let's compute the derivative of $[\bvt_k]^2$ in terms of time $t$:
  \begin{align*}
    \frac{\rd [\bvt_k]^2}{\rd t} 
    &= 2\bvt_k\cdot \frac{d}{dt}\frac{\vt_k}{\n{\vt}}\\
    &= 2\bvt_k\cdot \frac{1}{\n{\vt}}\frac{d}{dt}\vt_k 
      + 2[\bvt_k]^2\cdot \frac{d}{dt}\frac{1}{\n{\vt}}\\
    &= 2\bvt_k\cdot \frac{1}{\n{\vt}} 
      [-\nabla L(\vt)]_k - 2[\bvt_k]^2
      \cdot \frac{\inner{\bvt}{-\nabla L(\vt)}}{\n{\vt}}\\
    &= 2\bvt_k\cdot \frac{1}{\n{\vt}} [-(I-\bvt[\bvt]^\top)\nabla L(\vt)]_k.
  \end{align*}
  Note that 
  \[
    \nabla f(\vt) 
    = 4(\Tt-T^*)([\bvt]^{\otimes 2}, \bvt, I) 
      - 2(\Tt-T^*)([\bvt]^{\otimes 4})\bvt + \lambda \bvt,
  \]
  where the last two terms left multiplied by $(I-\bvt[\bvt]^\top)$ equals 
  to zero. Therefore,
  \[
    \frac{\rd [\bvt_k]^2}{\rd t} 
    = 8\bvt_k\br{(T^*-\Tt)([\bvt]^{\otimes 3)},I) 
      - (T^*-\Tt)([\bvt]^{\otimes 4)})\bvt }_k
  \]
  
  We can write $T^*$ as $\sum_{i\in[d]}a_i e_i^{\otimes 4}$ and write 
  $\Tt$ as $\sum_{i\in [d]} \Tt_i +\Tt_\varnothing$. Since Proposition~\ref{prop:main} 
  is true at time $t$, we know any $\wt$ in $\Wt_\varnothing$ has norm upper 
  bounded by $\delta_1$, which implies $\fn{\Tt_\varnothing}\leq m\delta_1^2$. 
  Therefore, we have 
  \[
    \absr{8\bvt_k\br{-\Tt_\varnothing([\bvt]^{\otimes 3)},I) +\Tt_\varnothing([\bvt]^{\otimes 4)})\bvt }_k } 
    \leq O(m\delta_1^2).
  \]
  
  For any $i\in [d],$ we have 
  \begin{align*}
    \br{\Tt_i([\bvt]^{\otimes 3}, I)}_k 
    =& \sum_{\wt\in \St_i}\ns{\wt} \inner{\bwt}{\bvt}^3 \bwt_k \\
    =& \hatt_k \Et_{k,w} \inner{\bwt}{\bvt}^3 \bwt_k,
  \end{align*}
  and 
  \begin{align*}
    \br{\Tt_i([\bvt]^{\otimes 4})\bvt}_k 
    =& \sum_{\wt\in \St_i}\ns{\wt} \inner{\bwt}{\bvt}^4 \bvt_k \\
    =& \hatt_k \Et_{k,w} \inner{\bwt}{\bvt}^4 \bvt_k.
  \end{align*}
  
  For any $i\in [d],$ we have 
  \begin{align*}
    \br{T^*([\bvt]^{\otimes 3}, I)}_k = [\bvt_k]^3 \indic{i=k}
  \end{align*}
  and 
  \begin{align*}
    \br{T^*([\bvt]^{\otimes 4})\bvt}_k = [\bvt_i]^4 \bvt_k
  \end{align*}
  
  Based on the above calculations, we can see that 
  \begin{align*}
    G_1 
    &= 8\bvt_k\br{(T^*_k-\Tt_k)([\bvt]^{\otimes 3)},I) 
      - (T^*_k-\Tt_k)([\bvt]^{\otimes 4)})\bvt }_k\\
    G_2 &= 8\bvt_k\br{\sum_{i\neq k}\Tt_i([\bvt]^{\otimes 3)},I)}_k\\
    G_3 &= 8[\bvt_k]^2 \sum_{i\neq k}(T^*_i-\Tt_i)([\bvt]^{\otimes 4)}) ,
  \end{align*}
  and the error term $O(m\delta_1^2)$ comes from $\Tt_\varnothing$. 
  To complete the proof, use the identity $\inner{\bw}{\bv} = \bw_k \bv_k 
  + \inner{\bw_{-k}}{\bv_{-k}}$ to rewrite $G_1$. 
\end{proof}

One may wish to skip all following estimations and come back to them when 
needed. 

\begin{lemma}
  \label{lemma: inner, bvk>=, any w}
  For any $\bv$ with $\bv_k^2 \ge 1 - \alpha$ and any $\bw \in 
  \mathbb{S}^{d-1}$, we have $|\inner{\bv}{\bw}| = |\bw_k| \pm \sqrt{\alpha}$.
\end{lemma}
\begin{proof}
  Assume w.o.l.g.~that $k=1$. Note that the set $\{\bv \in 
  \mathbb{S}^{d-1} \;:\; \bv_k^2 \ge 1 - \alpha \}$ is invariant under
  rotation of other coordinates, whence we may further assume w.o.l.g. that
  $\bw = \bw_1 e_1 + \sqrt{1 - \bw_1^2} e_2$. Then, 
  \begin{align*}
    |\inner{\bw}{\bv}|
    &= \left|\bw_1 \bv_1 
      + \sqrt{1 - \bv_1^2}\sqrt{1 - \bw_1^2} \right| \\
    &\ge |\bw_1| \sqrt{1 - \alpha} 
      - \sqrt{\alpha} \sqrt{1 - \bw_1^2} \\
    &= \frac{\bw_1^2 (1 - \alpha) - \alpha (1 - \bw_1^2)}
      {|\bw_1| \sqrt{1 - \alpha} 
        + \sqrt{\alpha} \sqrt{1 - \bw_1^2}} \\
    &= \frac{\bw_1^2 - \alpha}
      {|\bw_1| \sqrt{1 - \alpha} 
        + \sqrt{\alpha} \sqrt{1 - \bw_1^2}} 
    \ge \frac{\bw_1^2 - \alpha}{|\bw_1| + \sqrt{\alpha} } 
    = |\bw_1| - \sqrt{\alpha}. 
  \end{align*}
  The other direction follows immediately from
  \[
    |\inner{\bw}{\bv}|
    \le |\bw_1| |\bv_1| 
      + \left|\sqrt{1 - \bv_1^2}\sqrt{1 - \bw_1^2} \right| 
    \le |\bw_1| + \sqrt{\alpha}. 
  \]
\end{proof}

The next two lemmas bound the cross interaction between different 
$\St_k$. 

\begin{lemma}
  \label{lemma: cross interaction, individual}
  Suppose that at time $t$, Proposition~\ref{prop:main} is true.
  Then for any $\vt \in \St_k$ and $l \ne k$, the following hold.
  \begin{enumerate}[(a)]
    \item $[\bvt_l]^4 \le \alpha^2$.
    \item $\Et_{l, w}\left\{ [z_t]^4 \right\} \le O(\alpha^2)$.
    \item $\Et_{l, w}\left\{ [z_t]^3 \bv_l \bw_l \right\} \le O(\alpha^2)$.
  \end{enumerate}
\end{lemma}
\begin{proof}
  (a) follows immediately from $[\vt_l]^4 \le (1 - [\vt_l]^2) \le \alpha^2$.
  For (b), apply Lemma~\ref{lemma: inner, bvk>=, any w} and we get
  \[
    \Et_{l, w}\left\{ [z_t]^4 \right\}
    \le \Et_{l, w}\left\{ \left( |\bw_k| + \sqrt{\alpha}\right)^4 \right\}
    \le \Et_{l, w}\left\{ 
      [\bw_k]^4 
      + 4 |\bw_k|^3 \sqrt{\alpha} 
      + 6 [\bw_k]^2 \alpha
      + 4 |\bw_k| \alpha^{1.5} 
      + \alpha^2
      \right\}.
  \]
  For the first three terms, it suffices to note that 
  $\Et_{l, w}\left\{ [\bw_k]^2 \right\} \le \alpha^2$. For the fourth term,
  it suffices to additionally recall Jensen's inequality. Combine these together
  and we get $\Et_{l, w}\left\{ [z_t]^4 \right\} = O(\alpha^2)$. The proof
  of (b), \textit{mutatis mutandis}, yields (c). 
\end{proof}

\begin{lemma}
  \label{lemma: cross interaction, average}
  Suppose that at time $t$, Proposition~\ref{prop:main} is true. Then
  for any $k \ne l$, the following hold.
  \begin{enumerate}[(a)]
    \item $\Et_{k, v} [\bvt_l]^4 \le O(\alpha^3)$. 
    \item $\Et_{k, v} \Et_{l, w} [\zt]^4 \le O(\alpha^3)$.
    \item $\Et_{k, v} \Et_{l, w} \left\{ [\zt]^3 \bv_k \bw_k \right\} 
      \le O(\alpha^3)$.
  \end{enumerate}
\end{lemma}
\begin{proof}
  For (a), we compute
  \[
    \Et_{k, v} [\bvt_l]^4
    \le \Et_{k, v} \left\{ \left( 1 - [\bvt_k]^2 \right)^2 \right\}
    \le \alpha \Et_{k, v} \left\{1 - [\bvt_k]^2 \right\}
    \le O(\alpha^3),
  \]
  where the second inequality comes from the condition (\ref{itm: Ist, 
  individual}) of Proposition~\ref{prop:main} and the third from condition 
  (\ref{itm: Ist, average}) of Proposition~\ref{prop:main}. 
  Now we prove (b). (c) can be proved in a similar fashion. For simplicity,
  write $\xt = \inner{\bwt_{-l}}{\bvt_{-l}}$. Clear that $|\xt| \le 
  \sqrt{1 - [\bwt_l]^2}$ and by Jensen's inequality and condition 
  (\ref{itm: Ist, average}) of Proposition~\ref{prop:main}, 
  $\Et_{l, w} \sqrt{1 - [\bwt_l]^2} \le O(\alpha)$. 
   We compute
  \begin{align*}
    \Et_{k, v} \Et_{l, w} [\zt]^4
    = \Et_{k, v} \Et_{l, w} \bigg\{
      [\bwt_l]^4 [\bvt_l]^4
      & + 4 [\bwt_l]^3 [\bvt_l]^3 \xt
        + 6 [\bwt_l]^2 [\bvt_l]^2 [\xt]^2 \\
      & + 4 \bwt_l \bvt_l [\xt]^3
        + [\xt]^4
    \bigg\}.
  \end{align*}
  We bound each of these five terms as follows.
  \begin{align*}
    \Et_{k, v} \Et_{l, w} \left\{ [\bwt_l]^4 [\bvt_l]^4 \right\}
    &\le \Et_{k, v} [\bvt_l]^4 \le O(\alpha^3), \\
    \Et_{k, v} \Et_{l, w} \left\{ [\bwt_l]^3 [\bvt_l]^3 \xt \right\}
    &\le \Et_{k, v} [\bvt_l]^3 
      \Et_{l, w} \left\{ \sqrt{1 - [\bwt_l]^2} \right\}
      \le O(\alpha^3),  \\
    \Et_{k, v} \Et_{l, w} \left\{ [\bwt_l]^2 [\bvt_l]^2 [\xt]^2 \right\}
    &\le \Et_{k, v} [\bvt_l]^2 
      \Et_{l, w} \left\{ 1 - [\bwt_l]^2 \right\}
      \le O(\alpha^3), \\
    \Et_{k, v} \Et_{l, w} \left\{ \bwt_l \bvt_l [\xt]^3 \right\}
    &\le \Et_{k, v} \bvt_l 
      \Et_{l, w} \left\{ \left(1 - [\bwt_l]^2\right)^{1.5} \right\}
      \le O(\alpha^3), \\
    \Et_{k, v} \Et_{l, w} [\xt]^4
    &\le \Et_{l, w} \left\{ \left(1 - [\bwt_l]^2\right)^2 \right\}
     \le O(\alpha^3). 
  \end{align*}
  Combine these together and we complete the proof. 
\end{proof}

\begin{lemma}
  \label{lemma: Ez4 approx vk4}
  Suppose that at time $t$, Proposition~\ref{prop:main} is true. Then, for any 
  $\vt \in \St_k$, we have $\Et_{k, w} \left\{ [\zt]^4 \right\} = [\bvt_k]^4 \pm 
  O(\alpha^{1.5})$.
\end{lemma}
\begin{proof}
  For simplicity, put $\xt = \inner{\bwt_{-k}}{\bvt_{-k}}$. Note that
  $|\xt| \le \sqrt{1 - [\bvt_k]^2} \sqrt{1 - [\bwt_k]^2}
  \le \sqrt{\alpha} \sqrt{1 - [\bwt_k]^2}$. Then
  \begin{align*}
    \Et_{k, w} \left\{ [\zt]^4 \right\}
    = \Et_{k, w} \left\{ \left[\bwt_k \bvt_k +  \xt\right]^4 \right\} 
    = [\bvt_k]^4 \Et_{k, w} \left\{ [\bwt_k]^4 \right\}
      \pm O(1) \Et_{k, w} \xt.
  \end{align*}
  For the first term, note that 
  \begin{align*}
    \Et_{k, w} \left\{ [\bwt_k]^4 \right\}
    = 1 - \Et_{k, w} \left\{ (1 - [\bwt_k]^2)(1 + [\bwt_k]^2) \right\} 
    \ge 1 - 2 \alpha^2.
  \end{align*}
  For the second term, by Jensen's inequality, we have
  \[
    \left|\Et_{k, w} \xt\right| 
    \le \sqrt{\alpha \Et_{k, w}[1 - [\bwt_k]^2]} \le \alpha^{1.5}.
  \]
  Thus,
  \[
    \Et_{k, w} \left\{ [\zt]^4 \right\}
    = [\bvt_k]^4 \left(1 \pm 2\alpha^2 \right)
      \pm O(\alpha^{1.5})
    = [\bvt_k]^4 \pm O(\alpha^{1.5}).
  \]
\end{proof}

\begin{lemma}
  \label{lemma: Evw z4}
  Suppose that at time $t$, Proposition~\ref{prop:main} is true. Then 
  we have $\Et_{k, v, w}\left\{ [\zt]^4 \right\} \ge 1 - O(\alpha^2)$.
\end{lemma}
\begin{proof}
  For simplicity, put $\xt = \inner{\bwt_{-k}}{\bvt_{-k}}$. We have
  \begin{align*}
    \Et_{k, v, w} \left\{ [\zt]^4 \right\}
    &= \Et_{k, v, w} \left\{ \left(\bwt_k \bvt_k + \xt \right)^4 \right\} \\
    &\ge \Et_{k, v, w} \left\{ 
      [\bwt_k]^4 [\bvt_k]^4
      + [\bwt_k]^3 [\bvt_k]^3 x
      + \bwt_k \bvt_k x^3
    \right\}.
  \end{align*}
  Note that 
  \begin{equation}
    \label{eq: Ekvw wk3 vk3 x >= 0}
    \begin{split}
      \Et_{k, v, w} \left\{ [\bwt_k]^3 [\bvt_k]^3 x \right\}
      &= \sum_{i\ne k} \Et_{k, v, w} \left\{ [\bwt_k]^3 [\bvt_k]^3 
        \bwt_i \bvt_i \right\} \\
      &= \sum_{i\ne k} \left( \Et_{k, v, w} \left\{ [\bwt_k]^3 
        \bwt_i \right\} \right)^2 \ge 0.
    \end{split}
  \end{equation}
  Similarly, $\Et_{k, v, w} \left\{ \bwt_k \bvt_k x^3 \right\} \ge 0$ also
  holds. Finally, by Jensen's inequality, we have
  \begin{align*}
    \Et_{k, v, w} \left\{ [\zt]^4 \right\}
    &\ge \Et_{k, v, w} \left\{ 
      [\bwt_k]^4 [\bvt_k]^4 \right\} \\
    &= \left( \Et_{k, w} \left\{ [\bwt_k]^4 \right\} \right)^2 
    \ge \left( \Et_{k, w} \left\{ [\bwt_k]^2 \right\} \right)^4 
    \ge \left( 1 - \alpha^2 \right)^4 
    = 1 - O(\alpha^2).
  \end{align*}
\end{proof}

\subsection{Condition (\ref{itm: Ist, individual}): the individual bound}\label{sec:individual}

In this section, we show Lemma~\ref{lemma: individual bound}, which implies 
condition (~\ref{itm: Ist, individual}) of Proposition~\ref{prop:main} always holds.

\lemmaindividualbound*

\begin{proof}
  Recall the definition of $G_1$, $G_2$ and $G_3$ from Lemma~\ref{lemma: d vtk2}. 
  Now we estimate each of these three terms. 
  By Lemma~\ref{lemma: Ez4 approx vk4}, the first two terms of $G_1$ can 
  be lower bounded by $8 \tat \left( 1 - [\bvt_k]^2 \right) [\bvt_k]^4
  - O(\hatt_k \alpha^{1.5})$ and, for the third term, replace $|\zt|$ with $1$,
  and then, by the Cauchy-Schwarz inequality and Jensen's inequality, it is 
  bounded $O(\hatt_k \alpha^{1.5})$. By Lemma~\ref{lemma: cross interaction, 
  individual}, $G_2$ and $G_3$ can be bounded by $O(1) \sum_{i \ne k} 
  \hatt_i \alpha^2$. Thus,
  \begin{align*}
    \frac{\rd}{\rd t} [\bvt]^2
    &\ge 8 \tat \left( 1 - [\bvt_k]^2 \right) [\bvt_k]^4
      - O(1) \sum_{i=1}^d \hatt_k \alpha^{1.5}
      - O(m \delta_1^2) \\
    &\ge 8 \tat \left( 1 - [\bvt_k]^2 \right) [\bvt_k]^4
      - O\left( \alpha^{1.5}\right).
  \end{align*}
  
  Now suppose that $[\bvt_k]^2 = 1 - \alpha$. By Proposition~\ref{prop:main}, we have
  $\tat \ge \lambda/6$. Hence, 
  \begin{align*}
    \frac{\rd}{\rd t} [\bvt]^2
    \ge \lambda \alpha (1 - \alpha)^2 
      - O\left( \alpha^{1.5}\right)
    \ge \lambda \alpha 
      - O\left( \alpha^{1.5}\right).
  \end{align*}
\end{proof}

\subsection{Condition (\ref{itm: Ist, average}): the average bound}\label{sec:average}
\subsubsection*{Bounding the total amount of impulses}
Note that there are two sources of impulses. First, when $\hatt_k$ is larger, 
the correlation of the newly-entered components is $1 - \alpha$ instead of 
$1 - \alpha^2$ and, second, the reinitialization may throw some components 
out of $\St_k$. 

First we consider the first type of impulses. 
Suppose that at time $t$, $\hatt_k \ge \alpha$, $\Et_{k, w}\left\{ [\bwt_k]^2
\right\} = B$, and one particle $\vt$ enters $\St_k$. The deterioration of the 
average bound can be bounded as
\begin{align*}
  B - 
  \left( \frac{\hatt_k}{\hatt_k + \ns{\vt}} B 
    + \frac{\ns{\vt}}{\hatt_k + \ns{\vt}} (1 - \alpha) \right)
  &= \frac{\ns{\vt}}{\hatt_k + \ns{\vt}}\left( B - (1 - \alpha)\right) \\
  &\le \frac{\ns{\vt}}{\alpha} 2 \alpha \\
  &= 2 \ns{\vt}.
\end{align*}
Hence, the total amount of impulses caused by the entrance of new components
can be bounded by $2m \delta_1^2$.

Now we consider the reinitialization. Again, it suffices to consider the case
where $\hatt_k \ge \alpha$. Suppose that at time $t$, $\hatt_k \ge \alpha$,
$\Et_{k, w}\left\{ [\bwt_k]^2\right\} = B$ and one particle $\vt \in \St_k$
is reinitialized. By the definition of the algorithm, its norm is at most
$\delta_1$. Hence, The deterioration of the average bound can be bounded 
as\footnote{The second term is obtained by solving the equation 
$B = \frac{\hatt_k - \ns{\vt}}{\hatt_k} B' + \frac{\ns{\vt}}{\hatt_k}
[\bvt_k]^2$ for $B'$.}
\begin{align*}
  B - \frac{\hatt_k}{\hatt_k - \ns{\vt}} 
    \left( B - \frac{\ns{\vt}}{\hatt_k} [\bvt_k]^2 \right)
  &= \frac{\ns{\vt}}{\hatt_k - \ns{\vt}} \left( [\bvt_k]^2 - B\right) \\
  &\le \frac{\ns{\vt}}{\hatt_k} 2\alpha \\
  &\le 2 \ns{\vt}. 
\end{align*}
Since there are at most $m$ components, the amount of impulses caused by
reinitialization is bounded by $2m\delta_1^2$. 

Combine these two estimations together and we know that the total amount of
impulses is bounded by $4m\delta_1^2$. This gives the epoch correction 
term of condition (c). 

\subsubsection*{The average bound}
First we derive a formula for the evolution of $\Et_{k, w}
\left\{ [\bvt_k]^2\right\}$. 

\begin{lemma}
  \label{lemma: average bound, formula}
  For any $k$ with $\St_k \ne \varnothing$, we have
  \begin{align*}
    \frac{\rd}{\rd t} \Et_{k, v} [\bvt_k]^2 
    =& \Et_{k, v} \br{\frac{d}{dt}[\bvt_k]^2} \\
      +& 4 \Et_{k,v} \br{\pr{(T^*-\Tt)([\bvt]^{\otimes 4})}\pr{[\bvt_k]^2}}
      - 4\pr{\Et_{k,v}(T^*-\Tt)([\bvt]^{\otimes 4})}\pr{\Et_{k,v}[\bvt_k]^2}.
  \end{align*}
\end{lemma}
\begin{remark}
  The first term corresponds to the tangent movement and the two terms in
  the second line correspond to the norm change of the components.
\end{remark}
\begin{proof}
  Recall that 
  \[
    \Et_{k,v}[\bvt_k]^2 = \frac{1}{\hatt_k}\sum_{\vt\in \St_k}\ns{\vt}[\bvt_k]^2.
  \]
  Taking the derivative, we have 
  \begin{align*}
    \frac{d}{dt}\Et_{k,v}[\bvt_k]^2 
    =& \frac{1}{\hatt_k}\sum_{\vt\in \St_k}\ns{\vt}\pr{\frac{d}{dt}[\bvt_k]^2} 
      + \frac{1}{\hatt_k}\sum_{\vt\in \St_k}\pr{\frac{d}{dt}\ns{\vt}}[\bvt_k]^2\\
    &+ \pr{\frac{d}{dt}\frac{1}{\hatt_k}}\sum_{\vt\in \St_k}\ns{\vt}[\bvt_k]^2.
  \end{align*}
  The first term is just $\Et_{k, v}\frac{d}{dt}[\bvt_k]^2$. Denote $R(\bvt) 
  = 2(T^*-\Tt)([\bvt]^{\otimes 4})-\lambda.$ We can write the second term as 
  follows:
  \begin{align*}
    \frac{1}{\hatt_k}\sum_{\vt\in \St_k}\pr{\frac{d}{dt}\ns{\vt}}[\bvt_k]^2 
    =& \frac{1}{\hatt_k}\sum_{\vt\in \St_k}2R(\bvt)\ns{\vt} [\bvt_k]^2 \\
    =& 2\Et_{k,v}\br{R(\bvt)[\bvt_k]^2}
  \end{align*}
  Finally, let's consider $\frac{d}{dt}\frac{1}{\hatt_k}$ in the third term,
  \begin{align*}
    \frac{d}{dt}\frac{1}{\hatt_k} =& -\frac{1}{[\hatt_k]^2}\frac{d}{dt}\hatt_k\\
    =& -\frac{1}{[\hatt_k]^2}\frac{d}{dt}\sum_{\vt\in\St_k}\ns{\vt}\\
    =& -\frac{2}{[\hatt_k]^2}\sum_{\vt\in\St_k}R(\bvt)\ns{\vt}\\
    =& -\frac{2}{\hatt_k}\Et_{k,v}R(\bvt).
  \end{align*}
  Overall, we have 
  \begin{align*}
    \frac{d}{dt}\Et_{k,v}[\bvt_k]^2 =& \Et_{k,v}\br{\frac{d}{dt}[\bvt_k]^2}\\
    &+ 4\Et_{k,v}\br{\pr{(T^*-\Tt)([\bvt]^{\otimes 4})}\pr{[\bvt_k]^2}}
    - 4\pr{\Et_{k,v}(T^*-\Tt)([\bvt]^{\otimes 4})}\pr{\Et_{k,v}[\bvt_k]^2}
  \end{align*}
\end{proof}

\begin{lemma}[Bound for the average tangent speed]
  Suppose that $m\delta_1^2 = O(\alpha^3)$ and, at time $t$, 
  Proposition~\ref{prop:main} is true and $\St_k \ne \varnothing$. Then we have
  \[
    \Et_{k,v} \br{\frac{d}{dt}[\bvt_k]^2}
    \geq 8(a_k-\hatt_k)(1-\Et_{k,v}[\bvt_k]^2) - O(\alpha^3).
  \]
\end{lemma}
\begin{proof}
  Recall the definition of $G_1$, $G_2$ and $G_3$ from Lemma~\ref{lemma: d 
  vtk2}.
  \begin{itemize}
    \item \textbf{Lower bound for $\Et_{k, v} G_1$.}
    By \eqref{eq: Ekvw wk3 vk3 x >= 0}, we have 
    $\Et_{k, v, w}\left\{ [\zt]^3 \inner{\bw_{-k}}{\bv_{-k}} \right\} \ge 0$,
    whence can be ignored. Meanwhile, note that $\Et_{k, w}
    \left\{ [\zt]^4 \right\}  \le 1$.  Therefore, 
    \begin{align*}
      \Et_{k, v} G_1
      &\ge 8 a_k \Et_{k, v}\left\{
          \left( 1 - [\bvt_k]^2 \right) [\bvt_k]^4 \right\}
        - 8 \hatt_k \Et_{k, v}\left\{ 1 - [\bvt_k]^2 \right\}.
    \end{align*}
    For the first term, we compute
    \begin{align*}
      \Et_{k, v}\left\{
        \left( 1 - [\bvt_k]^2 \right) [\bvt_k]^4 \right\}
      &= \Et_{k, v}\left\{
          \left( 1 - [\bvt_k]^2 \right) 
          \left( 1 - \left(1 + [\bvt_k]^4\right) \right) \right\} \\
      &= \Et_{k, v}\left\{ 1 - [\bvt_k]^2 \right\}
        - \Et_{k, v}\left\{
          \left( 1 - [\bvt_k]^2 \right)^2 
          \left( 1 + [\bvt_k]^2 \right)\right\} \\
      &\ge \Et_{k, v}\left\{ 1 - [\bvt_k]^2 \right\}
        - 2 \Et_{k, v}\left\{
          \left( 1 - [\bvt_k]^2 \right)^2 \right\} \\
      &\ge \Et_{k, v}\left\{ 1 - [\bvt_k]^2 \right\} - O(\alpha^3). 
    \end{align*}
    Thus,
    \[
      \Et_{k, v} G_1
      \ge 8 \tat_k \Et_{k, v}\left\{ 1 - [\bvt_k]^2 \right\} 
        - O\left( \hatt_k \alpha^3 \right).
    \]
    
    \item \textbf{Upper bound for $\Et_{k, v} |G_2|$ and $\Et_{k, v} |G_2|$.}
    It follows from Lemma~\ref{lemma: cross interaction, average} 
    that both terms are $O(1)\sum_{i\ne k} \hatt_i \alpha^3$. 
  \end{itemize}
  Combine these two bounds together, absorb $m\delta_1^2$ into $O(\alpha^3)$,
  and we complete the proof.
\end{proof}

\begin{lemma}[Bound for the norm fluctuation]
  Suppose that at time $t$, Proposition~\ref{prop:main} is true and $\St_k \ne 
  \varnothing$. Then at time $t$, we have
  \[
    4\Et_{k,v} \br{\pr{(T^*-\Tt)([\bvt]^{\otimes 4})}\pr{[\bvt_k]^2}}
    - 4\pr{\Et_{k,v}(T^*-\Tt)([\bvt]^{\otimes 4})}\pr{\Et_{k,v}[\bvt_k]^2} 
    \geq -O(\alpha^3)
  \]
\end{lemma}
\begin{proof}
  We can express $(T^*-\Tt)([\bvt]^{\otimes 4})$ as follows:
  \begin{align*}
      &(T^*-\Tt)([\bvt]^{\otimes 4}) \\
      =& (a_k-\hatt_k)[\bvt_k]^4 + \hatt_k\pr{[\bvt_k]^4-\Et_{k,w}
      \inner{\bwt}{\bvt}^4} + \sum_{i\neq k}a_i[\bvt_i]^4 
      - \sum_{i\neq k}\hatt_i \Et_{i,w}\inner{\bwt}{\bvt}^4 \pm O(m\delta_1^2)
  \end{align*}
  It's clear that $\Et_{k,v} \sum_{i\neq k}a_i[\bvt_i]^4 = O(\alpha^3)$ and 
  $\Et_{k,v} \sum_{i\neq k}\hatt_i \Et_{i,w}\inner{\bwt}{\bvt}^4 = O(\alpha^3)$,
  so their influence can be bounded by $O(\alpha^3)$. Let's then focus on the
  first two terms in $(T^*-\Tt)([\bvt]^{\otimes 4})$.
  
  For the first term, we have
  \begin{align*}
    &4\Et_{k,v}(a_k-\hatt_k)[\bvt_k]^4 [\bvt_k]^2 
    - 4\Et_{k,v}(a_k-\hatt_k)[\bvt_k]^4 \Et_{k,v}[\bvt_k]^2\\
    =& 4(a_k-\hatt_k)\pr{ \Et_{k,v}[\bvt_k]^6
    - \Et_{k,v}[\bvt_k]^4 \Et_{k,v}[\bvt_k]^2} \geq 0.
  \end{align*}
  
  Let's now turn our focus to the second term. Denote $x =
  \inner{\bwt_{-k}}{\bvt_{-k}}$ and write $\inner{\bwt}{\bvt}^4 = 
  [\bwt_{k}]^4[\bvt_k]^4 + 4[\bwt]_{k}^3[\bvt_k]^3 x + O(x^2)$. Suppose 
  $m=\Et_{k,v}[\bvt_k]^2$, we know $m\in [1-O(\alpha^2), 1]$. We also know that 
  $[\bvt_k]^2\in[1-\alpha, 1]$ for every $\bvt \in \St_i,$ so we have 
  $|[\bvt_k]^2-m|=O(\alpha)$. We have 
  \begin{align*}
      \absr{\Et_{k,v}\Et_{k,w} ([\bvt_k]^2 -m)[\bvt_k]^4(1-[\bwt_k]^4) } 
      = O(\alpha^3)\\
      \absr{\Et_{k,v}\Et_{k,w} ([\bvt_k]^2 -m)(\bwt_k\bvt_k)^3 x} 
      = O(\alpha^3) \\
      \Et_{k,v}\Et_{k,w} x^2 = O(\alpha^4)\\
  \end{align*}
  Therefore,
  \begin{align*}
    &4\Et_{k,v}\br{\hatt_k\pr{[\bvt_k]^4-\Et_{k,w}
    \inner{\bwt}{\bvt}^4}[\bvt_k]^2}
    - 4\Et_{k,v}\hatt_k\pr{[\bvt_k]^4-\Et_{k,w}\inner{\bwt}{\bvt}^4}
    \Et_{k,v}[\bvt_k]^2\\
    \geq& -O(\hatt_k \alpha^3).
  \end{align*}
  
  Combining the bounds for all four terms, we conclude that 
  \begin{align*}
      4\Et_{k,v}\br{(T^*-\Tt)([\bvt]^{\otimes 4})[\bvt_k]^2}
      - 4\Et_{k,v}(T^*-\Tt)([\bvt]^{\otimes 4})\Et_{k,v}[\bvt_k]^2
      \geq -O(\alpha^3).
  \end{align*}
\end{proof}

\lemmaaveragebound*

\begin{proof}
  It suffices to combine the previous three lemmas together. 
\end{proof}

\subsection{Condition (\ref{itm: Ist, residual}): bounds for the residual}
\label{sec:residual}

In this section, we consider condition~(\ref{itm: Ist, residual}) of
Proposition~\ref{prop:main}.
Again, we need to estimate the derivative of $\tat_k$ when $\tat_k$ touches 
the boundary. 

\paragraph{On the impulses}
Similar to the average bound in condition (\ref{itm: Ist, average}), we need
to take into consideration the impulses. For the lower bound on $\tat_k$,
we only need to consider the impulses caused by the entrance of new components
since the reinitialization will only increase $\tat_k$. By Proposition~\ref{prop:main}
and Assumption~\ref{assumption: induction, oracle}, the total amount of
impulses is upper bounded by $m\delta_1^2$. At the beginning of epoch $s$,
we have $\tat_k \ge \lambda/6 - (s-1) m \delta_1^2$, which is guaranteed by
the induction hypothesis from the last epoch. (At the beginning of the
first epoch, we have $\tat_k = a_k$). Thus, following Lemma~\ref{lemma: 
continuity argument with impulses}, it suffices to show that $\frac{\rd}{\rd t} 
\tat_k > 0$ when $\tat_k \le \lambda/6$. The upper bound on $\tat_k$ can be proved in a similar 
fashion. The only difference is that now the impulses that matter are caused
by the reinitialization, the total amount of which can again be bounded by
$m \delta_1^2$. 
  
\begin{lemma}
  \label{lemma: d hattk}
  Suppose that at time $t$, Proposition~\ref{prop:main} is true and no impulses 
  happen at time $t$. Then we have
  \[
    \frac{1}{\hatt_k} \frac{\rd}{\rd t} \hatt_k
    = 2 \sum_{i=1}^d a_i \Et_{k, v} [\bvt_i]^4
      - 2 \sum_{i=1}^d \hatt_i \Et_{k, v} \Et_{i, w} [\zt]^4
      - \lambda 
      - O(m \delta_1^2).
  \]
\end{lemma}
\begin{proof}
  Recall that $\hatt_k = \sum_{\vt \in \St_k} \ns{\vt}$ and 
  Lemma~\ref{lemma: d |v|2} implies that
  \begin{align*}
    \frac{\rd}{\rd t} \ns{\vt}
    &= 2 \sum_{i=1}^d a_i \ns{\vt} [\bvt_i]^4
      - 2 \sum_{i=1}^d \hatt_i \ns{\vt} \Et_{i, w} \left\{ [\zt]^4 \right\} \\
      &\quad - \lambda \ns{\vt} - \ns{\vt} O(m \delta_1^2).
  \end{align*}
  Sum both sides and we complete the proof.
\end{proof}

\begin{lemma}
  \label{lemma: d hattk, upper bound}
  Suppose that at time $t$, Proposition~\ref{prop:main} is true and no impulses 
  happen at time $t$. Assume $\delta_1^2 = O(\alpha^2/m).$ Then we have
  \[
    \frac{1}{\hatt_k} \frac{\rd}{\rd t} \hatt_k
    \le 2\tat_k - \lambda + O(\alpha^2). 
  \]
  In particular, when $\tat_k \le \lambda/6$, we have $\frac{\rd}{\rd t} \hatt_k
  < 0$.
\end{lemma}
\begin{proof}
  By Lemma~\ref{lemma: d hattk}, we have
  \begin{align*}
    \frac{1}{\hatt_k} \frac{\rd}{\rd t} \hatt_k
    \le 2 a_k - 2 \hatt_k \Et_{k, v} \Et_{k, w} [\zt]^4
      + 2 \sum_{i \ne k} a_i \Et_{k, v} [\bvt_i]^4
      - \lambda.
  \end{align*}
  By Lemma~\ref{lemma: Evw z4}, we have
  \[
    2 a_k - 2 \hatt_k \Et_{k, v} \Et_{k, w} [\zt]^4
    \le 2 \tat_k + O( a_k \alpha^2 )
  \]
  For each term in the summation, we have 
  \[
    \Et_{k, v} [\bvt_i]^4
    \le \Et_{k, v} \left\{ \left(1 - [\bvt_k]^2\right)^2 \right\}
    \le \alpha \Et_{k, v} \left\{1 - [\bvt_k]^2 \right\}
    \le \alpha^3. 
  \]
  Thus,
  \begin{align*}
    \frac{1}{\hatt_k} \frac{\rd}{\rd t} \hatt_k
    &\le 2 \tat_k + O( a_k \alpha^2 )
      + 2 \sum_{i \ne k} a_i^2 \alpha^3 
      - \lambda \\
    &\le 2\tat_k - \lambda + O(\alpha^2). 
  \end{align*}
\end{proof}

\begin{lemma}
  \label{lemma: d hattk, lower bound}
  Suppose that at time $t$, Proposition~\ref{prop:main} is true. and no impulses 
  happen at time $t$. Then at time $t$, we have
  \[
    \frac{1}{\hatt_k} \frac{\rd}{\rd t} \hatt_k
    \ge 2 \tat_k - \lambda - O\left(\alpha^2 \right).
  \]
  In particular, when $\tat_k \ge \lambda$, we have $\frac{\rd}{\rd t} \hatt_k
  > 0$. 
\end{lemma}
\begin{proof}
  By Lemma~\ref{lemma: d hattk} (and the fact $\hatt_i \le a_i$), we have
  \begin{align*}
    \frac{1}{\hatt_k} \frac{\rd}{\rd t} \hatt_k
    \ge 2 a_k \Et_{k, v} [\bvt_k]^4 - 2 \hatt_k 
      - 2\sum_{i \ne k} a_i \Et_{k, v} \Et_{i, w} [\zt]^4
      - \lambda
      - O(m\delta_1^2).
  \end{align*}
  Note that $\Et_{k, v} [\bvt_k]^4 \ge 1 - O(\alpha^2)$, whence
  \[
    2 a_k \Et_{k, v} [\bvt_k]^4 - 2 \hatt_k 
    \ge 2 \tat_k - O\left( a_k \alpha^2 \right).
  \]
  For each term in the summation, by Lemma~\ref{lemma: cross interaction, 
  average}, we have $\Et_{k, v} \Et_{i, w} [\zt]^4 \le O(\alpha^3)$. Thus,
  \begin{align*}
    \frac{1}{\hatt_k} \frac{\rd}{\rd t} \hatt_k
    \ge 2 \tat_k - \lambda - O\left(\alpha^2 \right).
  \end{align*}
\end{proof}

\subsection{Counterexample}
\label{sec: induction, counterexample}
We prove Claim~\ref{clm:example} as follows.
\example*

\begin{proof}
Similar as in Lemma~\ref{lemma: d vtk2}, we can compute $\frac{\rd}{\rd t}\bv_k^2$ as follows,
\begin{align*}
    \frac{\rd}{\rd t}\bv_k^2 =& 8(1-\bv_k^2)\bv_k^4\\ 
    &- 8(1-\bv_k^2)\pr{\ns{v}\inner{\bv}{\bv}^4 + \ns{w}\inner{\bw}{\bv}^4} \\
    &+ 8\pr{\ns{w}\inner{\bw}{\bv}^3\inner{\bw_{-k}}{\bv_{-k}} + \ns{v}\inner{\bv}{\bv}^3\inner{\bv_{-k}}{\bv_{-k}}}.
\end{align*}
Since $\bv_k^2 = 1-\alpha, \bv_k = \bw_k$ and $\bv_{-k} = - \bw_{-k},$ we have $\inner{\bw}{\bv}^4, \inner{\bw}{\bv}^3 \geq 1-O(\alpha)$ and $\inner{\bw_{-k}}{\bv_{-k}}=-\alpha.$ Therefore, we have
\begin{align*}
    \frac{\rd}{\rd t}\bv_k^2 \leq 8\alpha - 8\alpha(\ns{v}+\ns{w}(1-O(\alpha)))-8\ns{w}(1-O(\alpha))\alpha + 8\ns{v}\alpha
\end{align*}
We have \begin{align*}
    \frac{\rd}{\rd t}\bv_k^2 \leq 8\alpha\pr{(1-\ns{w}-\ns{v})-\ns{w}(1-O(\alpha)) + \ns{v}}
    <0,
\end{align*}
where the last inequality assumes $\ns{w}+\ns{v}\in [2/3, 1]$ and $\ns{v},\alpha$ smaller than certain constant. 
\end{proof}

\section{Proofs for (Re)-initialization and Phase 1}\label{sec:proof_init_phase1}

We specify the constants that will be used in the proof of initialization (Section~\ref{sec:proof_init}) and Phase 1 (Section~\ref{sec:proof_phase1}). We will assume it always hold in the proof of Section~\ref{sec:proof_init} and Section~\ref{sec:proof_phase1}. We omit superscript $s$ for simplicity.
\begin{prop}[Choice of parameters]\label{assumption-phase1-param}
     The following hold with proper choices of constants $\gamma,c_e,c_\rho, c_{max}, c_t$
    
    \begin{enumerate}
        
        
        \item $t_1^\prime := \frac{c_t d}{8\beta \log d} \le t_1\le \frac{(1-\gamma)}{8\beta c_e}\cdot \frac{d}{\log d}$ ,

        \item $\Gamma_i=\frac{1}{8a_i t_1^\prime}$ if $S_i^{(s,0)}=\varnothing$, and $\Gamma_i = \frac{1}{8\lambda t_1^\prime}$ otherwise. $\rho_i=c_\rho \Gamma_i$. $\Gamma_{max}=c_{max}\log d/d$.
        
        \item $c_e < \frac{c_\rho c_{max}}{2(1-c_\rho)}$,  $c_\rho/c_t > 4c_e$, $c_tc_{max}\ge 4$.
        
        \item $c_a =(1-c_\rho)/(c_t c_{max})$
    \end{enumerate}
\end{prop}
\begin{proof}
    The results hold if let $\gamma,c_e,c_\rho,c_t$ be small enough constant and $c_{max}$ be large enough constant. For example, we can choose $c_e < c_\rho/4 < 0.01$, $c_t,\gamma<0.01$ and $c_{max}>10/c_t$.
\end{proof}

\subsection{Initialization}\label{sec:proof_init}

We give a more detailed version of initialization with specified constants to fit the definition of $S_{good}$, $S_{pot}$ and $S_{bad}$. We show that at the beginning of any epoch $s$, the following conditions hold with high probability. Intuitively, it suggests all directions that we will discover satisfy $a_i=\Omega(\beta)$ as $S_{i,pot}\neq \varnothing$. 
\begin{lemma}[(Re-)Initialization space]\label{assumption-phase1-init}
    In the setting of Theorem~\ref{thm:main}, the following hold at the beginning of current epoch with probability $1-1/\poly(d)$.
    \begin{enumerate}
        \item For all $a_i-\ha^{(0)}_i\ge \beta$, we have $S_{i,good}\ne\varnothing$.
        \item For all $a_i-\ha^{(0)}_i<\beta c_a$, we have $S_{i,pot}=\varnothing$.
        \item $S_{bad}=\varnothing$
        \item $\n{v^{(0)}}_2=\Theta(\delta_0)$, $[\bvO_i]^2\le \Gamma_{max}=c_{max}\log d/d$
        \item For every $v$, there are at most $O(\log d)$ many $i\in[d]$ such that $[\bvO_i]^2\ge c_e\log(d)/(10d)$.
        \item $|\{v|v \text{ was reinitialized in epoch $s$}\}|=(1-O(1/\log^2 d))m$.
    \end{enumerate}
\end{lemma}

\begin{proof}
    Let the constants in Lemma~\ref{lem-init-calculation} be $\eta=1/c_t$, $c_i = \Gamma_i d/\log d$ and satisfy Proposition~\ref{assumption-phase1-param}, then we know at the time of (re-)initialization, all statements hold. Since we further know from Lemma~\ref{lem:phase2} that $\n{v}=\Theta(\delta_0)$ and $\bv_i^2$ will only change $o(\log d/d)$, we have at the beginning of every epoch, all statements hold.
\end{proof}

\begin{lemma}\label{lem-init-calculation}
    There exist $m_0=\poly(d)$ and $m_1=\poly(d)$ such that if $m\in[m_0,m_1]$ and we random sample $m$ vectors $v$ from Unif$(\mathbb{S}^{d-1})$, with probability $1-1/\poly(d)$ the following hold with proper absolute constant $\eta$, $\gamma$, $c_\rho$, $c_i$, $c_e$, $c_{max}$ satisfying $\eta(1-\gamma)\le c_i$, $c_{max}\ge 4 \eta$, $\gamma,c_\rho$ are small enough and $c_{max}, \eta$ are large enough
    \begin{enumerate}
        \item For every $i\in[d]$ such that $c_i \le \eta$, there exists $v$ such that $[\bvO_i]^2\ge c_i(1+2c_\rho)\log d /d$ and $[\bvt_j]^2\le c_j(1-2c_\rho)\log d/d$ for $j\neq i$. 
        
        \item For every $v$, there does not exist $i\neq j$ such that $[\bvO_i]^2\ge c_i(1-2c_\rho)\log d /d$ and $[\bvO_j]^2\ge c_j(1-2c_\rho)\log d /d$.
        
        \item For every $v$ and $i\in[d]$, $[\bvO_i]^2\le c_{max}\log d /2d$.
        
        \item For every $v$, there are at most $O(\log d)$ many $i\in[d]$ such that $[\bvO_i]^2\ge c_e\log(d)/11d$.
        
        \item $|\{v|\text{there exists } i\in[d] \text{ such that } [\bvO_i]^2\ge c_i(1-2c_\rho)\log d/d\}| \le m/\log^2 (d).$
    \end{enumerate}
\end{lemma}
\begin{proof}
    It is equivalent to consider sample $v$ from $\mathcal{N}(0,I)$. Let $x\in \R$ be a standard Gaussian variable, according to Proposition 2.1.2 in \cite{vershynin2018high}, we have for any $t>0$
    \[\pr{\frac{2}{t}-\frac{2}{t^3}}\cdot \frac{1}{\sqrt{2\pi}}e^{-t^2/2}\leq \Pr\br{x^2 \geq t^2}\leq \frac{2}{t}\cdot \frac{1}{\sqrt{2\pi}}e^{-t^2/2}.\]

    Therefore, for any $i\in[d],$ we have for any constant $c>0$
    \[\Pr\br{v_i^2 \geq c\log(d)}=\Theta(d^{-c/2}\log^{-1/2} d).\]
    According to Theorem 3.1.1 in \cite{vershynin2018high}, we know with probability at least $1-2\exp(-\Omega(d)),$ $(1-r) d \leq \ns{v}\leq (1+r)d$ for any constant $0<r<1$. Hence, we have
    \[\Pr\br{\bv_i^2 \geq \frac{c\log(d)}{d}}\ge\Theta(d^{-c(1+r)/2}\log^{-1/2} d),\]
    \[\Pr\br{\bv_i^2 \geq \frac{c\log(d)}{d}}\le\Theta(d^{-c(1-r)/2}\log^{-1/2} d).\]
    \paragraph{Part 1.}
    For fixed $i\in[d]$ such that $\eta(1-\gamma)\le c_i \le \eta$, we have
    \[\Pr\br{\bv_i^2 \geq c_i(1+2c_\rho)\log(d)/d}\ge\Theta(d^{-c_i(1+2c_\rho)(1+r)/2}\log^{-1/2} d),\]
    
    For a given $j\neq i$, we have
    \begin{align*}
        &\Pr\br{\bv_i^2 \geq c_i(1+2c_\rho)\log(d)/d,\ \bv_j^2 \geq c_j(1-2c_\rho)\log(d)/d}\\ &\le\Theta(d^{-c_i(1+2c_\rho)(1-r)/2 -c_j(1-2c_\rho)(1-r)/2 })
        =O(d^{-\eta(1-\gamma)(1-r)}).
    \end{align*}
    Since $c_i\le\eta$, we know the desired event happens with probability $\Theta(d^{-\eta(1+2c_\rho)(1+r)/2}-d^{-\eta(1-\gamma)(1-r)+1})$. Since $\gamma,c_\rho$ are small enough constant, when $m_0\ge \Omega(d^{\eta(1+2c_\rho)(1+r)/2+1})$, with probability $1-O(e^{-d})$ there exists at least one $v$ such that $\bv_i^2 \geq c_i(1+2c_\rho)\log(d)$ and $[\bvt_j]^2\le c_j(1-2c_\rho)\log d/d$ for $j\neq i$. Take the union bound for all $i\in[d]$, we know when $m_0\ge \Omega(d^{\eta(1+2c_\rho)(1+r)/2+2})$, the desired statement holds with probability $1-O(de^{-d})$.

    \paragraph{Part 2.}
    For any given $i\neq j$, we have
    \[\Pr\br{[\bvO_i]^2\ge c_i(1-2c_\rho)\log d /d,\ [\bvO_j]^2\ge c_j(1-2c_\rho)\log d /d}
    \le O(d^{-(c_i+c_j)(1-2c_\rho)(1-r)/2}).\]
    Since $\eta(1-\gamma)\le c_i$, the probability that there exist $i\neq j$ such that the above happens is at most $O(d^{-\eta(1-\gamma)(1-2c_\rho)(1-r)+2})$. Thus, with $m_1\le O(d^{\eta(1-\gamma)(1-2c_\rho)(1-r)-2}/\poly(d))$, the desired statement holds with probability $1-1/\poly(d)$.
        
    \paragraph{Part 3.} We know
    \[\Pr\br{\text{for all } i\in[d],\ \bv_i^2\le c_{max}\log d/2d }\ge 1-O(d^{-c_{max}(1-r)/4+1}).\]
    With $m_1\le O(d^{c_{max}(1-r)/4-1}/\poly(d))$ the desired statement holds with probability $1-1/\poly(d)$.
    
    \paragraph{Part 4.} Since $m\le m_1=\poly(d)$, we know for any constant $c_e$, this statement holds with probability $1-O(e^{-\log^2 d})$.
    
    \paragraph{Part 5.}
    We have
    \[\Pr\br{\text{there exists } i\in[d] \text{ such that } [\bvO_i]^2\ge c_i(1-2c_\rho)\log d/d} \le O(d^{-c_i(1-2c_\rho)/2+1}).\]
    Let $p$ be the above probability and set $A$ as the $v$ satisfy above condition, by Chernoff's bound we have
    \[\Pr\br{|A|\ge m/\log^2 d}\le e^{-pm}\left(\frac{epm}{m/\log^2 d}\right)^{m/\log^2 d}=O(e^{-d}).\]
    
    Combine all parts above, we know as long as $r,\gamma,c_\rho$ are small enough, $c_{max}\ge 4\eta$ and $\eta$ is large enough, we have when $m_0\ge\Omega(d^{0.6\eta })$ and $m_1\le O(d^{0.9\eta})$, the results hold.
\end{proof}

\subsection{Proof of Phase 1}\label{sec:proof_phase1}
In this section, we first give a proof overview of Phase 1 and then give the detailed proof for each lemma in later subsections.

\subsubsection{Proof overview}\label{sec:phase1}


We give the proof overview in this subsection and present the proof of Lemma~\ref{lem-phase1-summary-trajectory} and Lemma~\ref{lem:phase1} at the end of this subsection. We remark that the proof idea in this phase is inspired by \citep{li2020learning}. 

We describe the high-level proof plan for phase 1. Recall that at the beginning of this epoch, we know $S_{bad}=\varnothing$ which implies there is at most one large coordinate for every component. Roughly speaking, we will show that for those small coordinate they will remain small in phase 1, and the only possibility for one component to have larger norm is to grow in the large direction. This intuitively suggests all components that have a relatively large norm in phase 1 are basis-like components.



We first show within $t_1^\prime = c_t d/(8\beta\log d))$ time, there are components that can improve their correlation with some ground truth component $e_i$ to a non-trivial $\polylog(d)/d$ correlation. This lemma suggests that there is at most one coordinate can grow above $O(\log d/d)$. 

Note that we should view the analysis in this section and the analysis in Appendix~\ref{sec: appendix, induction hypothesis} as a whole induction/continuity argument. It's easy to verify that at any time $0\leq t\leq t_1^{(s)}$, Assumption~\ref{assumption: induction, oracle} holds and Proposition~\ref{prop:main} holds.
        
        

\begin{lemma}\label{lem-phase1-lottery}
    In the setting of Lemma~\ref{lem:phase1}, suppose $\ninf{\bvO}^2\le \log^4(d)/d$. Then, for every $k\in[d]$
    \begin{enumerate}
        \item for $v\not\in S_{pot}$, $[\bvt_i]^2=O(\log(d)/d)$ for all $i\in[d]$ and $t\le t_1^\prime$.
        
        \item if $\St_k=\varnothing$ for $t\le t_1^\prime$, then for $v\in S_{k,good}$, there exists $t\le t_1^\prime$ such that $[\bvt_k]^2 \ge \log^4(d)/d$ and $[\bvt_i]^2= O(\log(d)/d)$ for all $i\ne k$.
        
        \item for $v\in S_{k,pot}\setminus (S_{good}\cup S_{bad})$, $[\bvt_i]^2= O(\log(d)/d)$ for all $i\ne k$ and $t\le t_1^\prime$.
    \end{enumerate}
\end{lemma}

The above lemma is in fact a direct corollary from the following lemma when considering the definition of $S_{good}$ and $S_{pot}$. It says if a direction is below certain threshold, it will remain $O(\log d/d)$, while if a direction is above certain threshold and there are no basis-like components for this direction, it will grow to have a $\polylog(d)$ improvement.
\begin{restatable}{lemma}{lemphaseonepolyloggap}\label{lem-phase1-polyloggap}
    In the setting of Lemma~\ref{lem:phase1}, we have
    \begin{enumerate}
        \item if $[\bvO_k]^2\le \min\{\Gamma_k-\rho_k,\Gamma_{max}\}$, then $[\bvt_k]^2=O( \log(d)/d)$ for $t\le t_1^\prime$.
    
        \item if $\St_k=0$ for $t\le t_1^\prime$, $[\bvO_k]^2\ge \Gamma_k+\rho_k$, $[\bvO_i]^2\le \Gamma_i-\rho_i$ for all $i\ne k$ and $\ninf{\bvO}^2\le \log^4(d)/d$, then there exists $t\le t_1^\prime$ such that $[\bvt_k]^2 \ge \log^4(d)/d$.
    \end{enumerate}
\end{restatable}

The following lemma shows if $[\bvx{t_1^\prime}_i]^2=O(\log d/d)$ at $t_1^\prime$, it will remain $O(\log d/d)$ to the end of phase 1. 
This implies for components that are not in $S_{pot}$, they will not have large correlation with any ground truth component in phase 1.

\begin{restatable}{lemma}{lemphaseoneremainsmall}\label{lem-phase1-remainsmall}
    In the setting of Lemma~\ref{lem:phase1}, suppose  $[\bvx{t_1^\prime}_i]^2= O(\log(d)/d)$. Then we have $[\bvt_i]^2= O(\log(d)/d)$ for $t_1^\prime\le t\le t_1$.
\end{restatable}


The following two lemmas show good components (those have $\polylog(d)/d$ correlation before $t_1^\prime$) will quickly grow to have constant correlation and $\delta_1$ norm. Note that the following condition $a_k=\Omega(\beta)$ holds in our setting because when $a_i<\beta c_a$, we have $S_{i,good}=S_{i,pot}=\varnothing$ (this means for those small directions there are no components that can have $\polylog(d)/d$ correlation as shown in Lemma~\ref{lem-phase1-lottery}).
\begin{restatable}[Good component, constant correlation]{lemma}{lemphaseoneconstantgap}\label{lem-phase1-constantgap}
    In the setting of Lemma~\ref{lem:phase1}, suppose $\St_k=\varnothing$ for $t\le t_1$, $a_k=\Omega(\beta)$. If there exists $\tau_0\le t_1$ such that $[\bvx{\tau_0}_k]^2> \log^4(d)/d$ and $[\bvx{\tau_0}_i]^2= O(\log(d)/d)$ for all $i\ne k$, then for any constant $c\in (0,1)$ we have $[\bvt_k]^2> c$ and $[\bvt_i]^2= O(\log(d)/d)$ for all $i\ne k$ when $\tau_0+t_1^\dbprime \le t\le t_1$ with $t_1^\dbprime=\Theta(d/(\beta \log^3 d))$. 
\end{restatable}


\begin{restatable}[Good component, norm growth]{lemma}{lemphaseonenormgrow}\label{lem-phase1-normgrow}
     In the setting of Lemma~\ref{lem:phase1}, suppose $\St_k=\varnothing$ for $t\le t_1$, $a_k=\Omega(\beta)$. If there exists $\tau_0^\prime\le t_1$ such that $[\bvx{\tau_0^\prime}_k]^2> c$ and $[\bvx{\tau_0^\prime}_i]^2= O(\log(d)/d)$ for all $i\ne k$, then we have $\n{\vt}_2\ge \delta_1$ for some  $\tau_0^\prime \le t\le \tau_0^\prime+t_1^\triprime$ with $t_1^\triprime=\Theta(\log(d/\alpha)/\beta)$. 
\end{restatable}

Recall from Lemma \ref{lem-phase1-polyloggap} we know there is at most one coordinate that can be large. Thus, intuitively we can expect if the norm is above certain threshold, the component will become basis-like, since this large direction will contribute most of the norm and other directions will remain small. In fact, we can show (1) norm of ``small and dense'' components (e.g., those are not in $S_{pot}$) is smaller than $\delta_1$; (2) once a component reaches norm $\delta_1$, it is a basis-like component.
        
        

        

\begin{restatable}{lemma}{lemphaseonenormthreshold}\label{lem-phase1-normthreshold}
    In the setting of Lemma~\ref{lem:phase1}, we have
    \begin{enumerate}
        \item if $\ninf{\bvt}^2\le\log^4 (d)/d$ for all $t\le t_1$, then $\n{\vt}_2=O(\delta_0)$ for all $t\le t_1$.
        
        \item Let $\tau_0=\inf\{t\in[0,t_1]|\ninf{\bvt}^2\ge \log^4 d/d\}$. Suppose $[\bvx{\tau_0}_k]^2\ge \log^4 d/d$ and $[\bvx{\tau_0}_i]^2=O( \log d/d)$ for $i\ne k$. If there exists $\tau_1$ such that $\tau_0 < \tau_1 \le t_1$ and $\n{\vx{\tau_1}}_2\ge \delta_1$ for the first time, then there exists $k\in[d]$ such that $[\bvx{\tau_1}_k]^2\ge 1-\alpha^2$ if $\hatt_k\le \alpha$ for $t\le \tau_1$ and $[\bvx{\tau_1}_k]^2\ge 1-\alpha$ otherwise.
    \end{enumerate}
\end{restatable}

One might worry that a component can first exceeds the $\delta_1$ threshold then drop below it and eventually gets re-initialized. Next, we show that re-initialization at the end of Phase 1 cannot remove all the components in $S^{(t_1)}_k.$

\begin{restatable}{lemma}{lemabovedelta}\label{lem:above_delta1}
If $S^{(0)}_k = \varnothing$ and $S^{(t')}_k \neq \varnothing$ for some $t'\in (0, t_1]$, we have $S^{(t_1)}_k\neq \varnothing$ and $\hat{a}_k^{(t_1)}\geq \delta_1^2.$
\end{restatable}

Given above lemma, we now are ready to prove Lemma~\ref{lem-phase1-summary-trajectory} and the main lemma for Phase 1.

\lemphaseonesummarytrajectory*

\begin{proof}
    We show statements one by one.
    \paragraph{Part 1.} The statement follows from Lemma~\ref{lem-phase1-lottery}, Lemma~\ref{lem-phase1-remainsmall} and Lemma~\ref{lem-phase1-normthreshold}.
    \paragraph{Part 2.}
    Suppose $\St_k=\varnothing$ for all $t\le t_1$. By Lemma~\ref{assumption-phase1-init} we know $S_{k,good}\neq \varnothing$. Then by Lemma~\ref{lem-phase1-lottery}, Lemma~\ref{lem-phase1-constantgap} and Lemma~\ref{lem-phase1-normgrow}, we know there exists $v$ such that $\n{\vt}_2\ge \delta_1$ within time $t_1=t_1^\prime + t_1^\dbprime + t_1^\triprime$. Then by Lemma~\ref{lem-phase1-normthreshold} we know $[\bvt_k]^2\ge 1-\alpha$. Therefore, we know there exists $t\le t_1$ such that $\St_k\neq\varnothing$. Finally we know it will keep until $t_1$ by Lemma~\ref{lem:above_delta1}.

    
    \paragraph{Part 3.} 
    The statement directly follows from Lemma~\ref{lem-phase1-normthreshold} and Lemma~\ref{lem:above_delta1}.
\end{proof}

\lemphaseone*

\begin{proof}

    By Lemma~\ref{assumption-phase1-init} we know the number of reinitialized components are always $\Theta(m)$ so Lemma~\ref{assumption-phase1-init} holds with probability $1-1/\poly(d)$ for every epoch. In the following assume Lemma~\ref{assumption-phase1-init} holds. The second and third statement directly follow from Lemma~\ref{assumption-phase1-init} and Lemma~\ref{lem-phase1-summary-trajectory} as $S_{k,pot}=\varnothing$ when $a_k\le \beta c_a$. For the first statement, combing the proof in Appendix~\ref{sec: appendix, induction hypothesis} and Lemma~\ref{lem-phase1-normthreshold}, we know the statement holds (see also the remark at the beginning of Appendix~\ref{sec: appendix, induction hypothesis}).
\end{proof}


        
        

\subsubsection{Preliminary}
To simplify the proof in this section, we introduce more notations and give the following lemma.
\begin{lemma}
    In the setting of Lemma~\ref{lem:phase1}, we have
        $T^* - \Tt =\sum_{i\in [d]} \tat_i e_i^{\otimes 4} + \Deltat$, where $\tat_i=a_i-\hatt_i$ and $\n{\Delta}_F =O(\alpha+m\delta_1^2)$. We know $\ta^{(0)}_i=a_i$ if $S_i^{(s,0)}= \varnothing$ and $\tat_i=\Theta(\lambda)$ if $S_i^{(s,0)}\neq \varnothing$. That is, the residual tensor is roughly the ground truth tensor $T^*$ with unfitted directions at the beginning of this epoch and plus a small perturbation $\Delta$.
        
\end{lemma}

\begin{proof}
    We can decompose $\Tt$ as 
    \[\Tt = \sum_{i\in[d]}\Tt_i + \Tt_\varnothing= \sum_{i\in[d]}\left(\hatt_ie_i^{\otimes 4} + (\Tt_i-\hatt_ie_i^{\otimes 4})\right)+\Tt_\varnothing,\]
    where $\Tt_i=\sum_{w\in S_i^{(t)}} \n{w}^2\bar{w}^{\otimes 4}$ and $\Tt_\varnothing=\sum_{w\in S_\varnothing^{(t)}} \n{w}^2\bar{w}^{\otimes 4}$. Note that when $S_i^{(t)}=\varnothing$, $\hatt_i=0$ and when $S_i^{(t)}\neq\varnothing$ we have $\fn{(\Tt_i-\hatt_ie_i^{\otimes 4})}=O(\hatt_i\alpha)$ and $\fn{\Tt_\varnothing}\leq m\delta_1^2.$ This gives the desired form of $T^*-\Tt$.
    
\end{proof}

We give the dynamic of $[\bvt_k]^2$ and $[\vt_k]^2$ here, which will be frequently used in our analysis. 
\begin{equation}\label{eq-dynamic-bvt}
\begin{aligned}
    \frac{\rd [\bvt_k]^2}{\rd t} 
    &= 2\bvt_k\cdot \frac{d}{dt}\frac{\vt_k}{\n{\vt}}\\
    &= 2\bvt_k\cdot \frac{1}{\n{\vt}}\frac{d}{dt}\vt_k + 2[\bvt_k]^2\cdot \frac{d}{dt}\frac{1}{\n{v}}\\
    &= 2\bvt_k\cdot \frac{1}{\n{\vt}} [-\nabla L(\vt)]_k - 2[\bvt_k]^2\cdot \frac{\inner{\bvt}{-\nabla L(\vt)}}{\n{\vt}}\\
    &= 2\bvt_k\cdot \frac{1}{\n{\vt}} [-(I-\bvt[\bvt]^\top)\nabla L(\vt)]_k\\
    &= 8\bvt_k\br{(T^*-\Tt)([\bvt]^{\otimes 3)},I) - (T^*-\Tt)([\bvt]^{\otimes 4)})\bvt }_k\\
    &= 8[\bvt_k]^2\left(\tat_k[\bvt_k]^2 - \sum_{i\in [d]} \tat_i[\bvt_i]^4 \pm \frac{\n{\Deltat}_F}{|\bvt_k|}\right).
\end{aligned}    
\end{equation}

\begin{equation}\label{eq-dynamic-vt}
\begin{aligned}
    \frac{\rd [\vt_k]^2}{\rd t} 
    &= 2\vt_k\cdot \frac{\rd \vt_k}{\rd t}\\
    &= 2\vt_k\cdot [-\nabla L(\vt)]_k \\
    &= 4\vt_k\br{2(T^*-\Tt)([\bvt]^{\otimes 3)},I)\n{\vt}_2 - (T^*-\Tt)([\bvt]^{\otimes 4)})\vt }_k\\
    &= 4[\vt_k]^2 \left(2\tat_k [\bvt_k]^2 - \sum_{i\in[d]}\tat_i [\bvt_i]^4 \pm\frac{\n{\Deltat}_F\n{\vt}_2}{|\vt_k|}\right).
\end{aligned}
\end{equation}

The following lemma allows us to ignore these already fitted direction as they will remain as small as their (re-)initialization in phase 1. 
\begin{lemma}\label{lem-phase1-fitteddir}
    In the setting of Lemma~\ref{lem:phase1}, if direction $e_k$ has been fitted before current epoch (i.e., $S_k^{(s,0)}\ne \varnothing$), then for $v$ that was reinitialized in the previous epoch, we have $[\bvt_k]^2=O( \log(d)/d)$ for all $t\le t_1.$  
\end{lemma}

\begin{proof}
    Since direction $e_k$ has been fitted before current epoch, we know $\tat_k = \Theta(\lambda)$. We only need to consider the time when $[\bvt_k]^2\ge \log d/d$. By \eqref{eq-dynamic-bvt} we have 
    \begin{align*}
    \frac{\rd [\bvt_k]^2}{\rd t} 
    &= 8[\bvt_k]^2\left(\tat_k[\bvt_k]^2 - \sum_{i\in [d]} \tat_i[\bvt_i]^4 \pm \frac{\n{\Deltat}_F}{|\bvt_k|}\right)
    \le [\bvt_k]^2 O\left(\lambda+ d\n{\Deltat}_F\right).
    \end{align*}
    Since $\lambda$ and $\n{\Deltat}_F=O(\alpha+m\delta_1^2)$ are small enough and $[\bvO_k]^2=O(\log d/d)$, we know $[\bvt_k]^2=O(\log d/d)$ for $t\le t_1$.
    
\end{proof}

\subsubsection{Proof of Lemma~\ref{lem-phase1-lottery} and Lemma~\ref{lem-phase1-polyloggap}}
Lemma~\ref{lem-phase1-lottery} directly follows from Lemma~\ref{lem-phase1-polyloggap} and the definition of $S_{good}$, $S_{pot}$ and $S_{bad}$ as in Definition~\ref{def-phase1-partition}. We focus on Lemma~\ref{lem-phase1-polyloggap} in the rest of this section. We need following lemma to give the proof of Lemma \ref{lem-phase1-polyloggap}.

\begin{lemma}\label{lem-phase1-bv4}
In the setting of Lemma~\ref{lem:phase1}, if $\ninf{\bvt}^2\le \log^4(d)/d$, we have $\sum_i [\bvt_i]^4\le c_e\log d/d$ for all $t\le t_1$. 
\end{lemma}
\begin{proof}
    We claim that for all $t\le t_1$, there are at most $O(\log d)$ many $i\in[d]$ such that $[\bvt_i]^2\ge c_e\log(d)/2d$. Based on this claim, we know
    \begin{align*}
        \sum_{i\in[d]} [\bvt_i]^4 \le O(\log d) \frac{\log^8d}{d^2} + \sum_{i:[\bvt_i]^2<c_e\log(d)/2d } [\bvt_i]^4
        \le O\left(\frac{\log^9d}{d^2}\right) + \frac{c_e\log(d)}{2d}
        \le \frac{c_e\log(d)}{d},
    \end{align*}
    which gives the desired result.
    
    In the following, we prove the above claim. From Lemma~\ref{assumption-phase1-init}, we know when $t=0$, the claim is true. For any $[\bvO_k]^2\le c_e\log(d)/10d$, we will show $[\bvt_k]^2\le c_e\log(d)/2d$ for all $t\le t_1$. By \eqref{eq-dynamic-bvt} we have
    \begin{align*}
    \frac{\rd [\bvt_k]^2}{\rd t} 
    &= 8[\bvt_k]^2\left(\tat_k[\bvt_k]^2 - \sum_{i\in [d]} \tat_i[\bvt_i]^4 \pm \frac{\n{\Deltat}_F}{|\bvt_k|}\right).
    \end{align*}
    
    In fact, we only need to show that for any $\tau_0$ such that $[\bar{v}^{(\tau_0)}_k]^2= c_e\log(d)/10d$ and $[\bvt_k]^2\ge c_e\log(d)/10d$ when $\tau_0\le t \le \tau_0+t_1$, we have $[\bvt_k]^2\le c_e\log(d)/2d$. To show this, we have
    \begin{align*}
    \frac{\rd [\bvt_k]^2}{\rd t} 
    \le 8[\bvt_k]^2\left(\tat_k[\bvt_k]^2 + \frac{\n{\Deltat}_F}{|\bvt_k|}\right)
    \le [\bvt_k]^2 \cdot 16\tat_k[\bvt_k]^2
    \le [\bvt_k]^2 \cdot \frac{\beta}{1-\gamma}\cdot \frac{8c_e\log(d)}{d},
    \end{align*}
    where we use $\n{\Deltat}_F=O(\alpha+m\delta_1^2)$ 
    and $\tat_k\le \beta/(1-\gamma)$.
    Therefore, with our choice of $t_1$
    , we know $[\bvt_k]^2\le c_e\log(d)/2d$. This finish the proof.
    
\end{proof}

We now are ready to give the proof of Lemma \ref{lem-phase1-polyloggap}.

\lemphaseonepolyloggap*

\begin{proof}
    
    We focus on the dynamic of $[\bvt_k]^2$. For those already fitted direction $e_k$, we have $\Gamma_k=1/(8\lambda t_1^\prime)$, which means $\Gamma_{max}\le \Gamma_k-\rho_k$. From Lemma~\ref{lem-phase1-fitteddir} we know $[\bvt_k]^2=O(\log d/d)$ for $t\le t_1^\prime$. In the rest of proof, we focus on these unfitted direction $e_k$. By \eqref{eq-dynamic-bvt} we have
    \begin{align*}
    \frac{\rd [\bvt_k]^2}{\rd t} 
    &= 8[\bvt_k]^2\left(\tat_k[\bvt_k]^2 - \sum_{i\in [d]} \tat_i[\bvt_i]^4 \pm \frac{\n{\Deltat}_F}{|\bvt_k|}\right)
    \end{align*}

    \paragraph{Part 1.}
    Define the following dynamics $\pt$,
    \begin{align*}
        \frac{\rd \pt}{\rd t} &= 8\pt \left(a_k\pt + \frac{a_k c_e\log d}{d}\right),\quad \pO=[\bvO_k]^2
    \end{align*}
    
    Given that $\tat_i\le a_i$ and $\n{\Deltat}_F=O(\alpha+m\delta_1^2)$ is small enough
    , it is easy to see $[\bvt_k]^2\le\max\{\log(d)/d,\pt\}$. Then it suffices to bound $\pt$ to have a bound for $[\bvt_k]^2$. Consider the following dynamic $x^{(t)}$
    \begin{equation}\label{eq-xdynamic}
    \begin{aligned}
        \frac{\rd x^{(t)}}{\rd t} = \tau_1 [x^{(t)}]^2,\quad x^{(0)} = \tau_2.
    \end{aligned}
    \end{equation}
    We know $x^{(t)} = 1/(1/\tau_2 - \tau_1 t)$. Set $\tau_1 = 8a_k$ and $\tau_2=1/(\tau_1 t_1^\prime)=\Gamma_k$. Then, with our choice of $\rho_k=c_\rho\Gamma_k$, we know
    \begin{enumerate}
        \item $\pO=[\bvO_k]^2\le \Gamma_k-\rho_k\le \Gamma_{max}$. As long as $\rho_k\ge \frac{2c_e\log d}{d}$ and $x^{(0)}=\pO+\rho_k/2$, we have $\pt\le x^{(t)}-\rho_k/2$ for $t\le t_1^\prime$. Therefore, $p^{(t_1^\prime)}\le x^{(t_1^\prime)}\le 2\Gamma_k^2/\rho_k=O(\log d/d)$. 
        
        \item $\pO=[\bvO_k]^2\le \Gamma_{max} < \Gamma_k-\rho_k$. As long as $x^{(0)}=\pO+\frac{c_e\log d}{d}$, we have $\pt\le x^{(t)}-\frac{c_e\log d}{d}$ for $t\le t_1^\prime$. Therefore, $p^{(t_1^\prime)}\le x^{(t_1^\prime)}=O(\log d/d)$. 
    \end{enumerate} 
    Together we know $[\bvt_k]^2=O(\log d/d)$ for $t\le t_1^\prime$.

    \paragraph{Part 2.}
    Define the following dynamics $\qt$,
    \begin{align*}
        \frac{\rd \qt}{\rd t} &= 8\qt \left(a_k\qt - \frac{2\beta c_e\log d}{d}\right),\quad \qO=[\bvO_k]^2.
    \end{align*}
    
    Since $\St_k=\varnothing$, we know $\tat_k=a_k$. Given that $\n{\Deltat}_F=O(\alpha+m\delta_1^2)$ and Lemma \ref{lem-phase1-bv4}, it is easy to see as long as $\ninf{\bvt}^2\le \log^4 d/d$, if $\qO\ge[\bvO_k]^2\ge\Theta(\log d/d)$ and $a_k[\qO]^2 -\frac{2\beta c_e\log d}{d}>0$, we have $[\bvt_k]^2\ge\qt$. Then it suffices to bound $\qt$ to get a bound on $[\vt_k]^2$. Consider the same dynamic \eqref{eq-xdynamic} with same $\tau_1$ and $\tau_2$, as long as $\qO=[\bvO_k]^2\ge \Gamma_k+\rho_k$, $\rho_k\ge \frac{4\beta c_e\log d}{a_k d}$ and $x^{(0)}=\qO-\rho_k/2$, we have $\qt\ge x^{(t)}+\rho_k/2$ if $\ninf{\bvt}^2\le \log^4 d/d$ holds. We can verify that $x^{(T_1^\prime)} = +\infty$, which implies there exists $t\le t_1^\prime$ such that $\ninf{\bvt}^2 > \log^4 d/d$.

\end{proof}

\subsubsection{Proof of Lemma \ref{lem-phase1-remainsmall}}
\lemphaseoneremainsmall*

\begin{proof}
    Recall $t_1-t_1^\prime=t_1^\dbprime+t_1^\triprime=o(d/(\beta\log d))$, it suffices to show if $[\bvx{t_1^\prime}_i]^2= c_1\log(d)/d$, then $[\bvt_i]^2$ will be at most $2c_1\log(d)/d$ in $t_{max}^\prime=o(d/(\beta \log d))$ time. Suppose there exists time $\tau_1\le t_{max}^\prime$ such that $[\bvx{\tau_1}_i]^2\ge 2c_1\log(d)/d$ for the first time. We only need to show if $[\bvt_i]^2\ge c_1\log(d)/d$ for $t\le \tau_1$, we have $[\bvt_i]^2< 2c_1\log(d)/d$. We know the dynamic of $[\bvt_i]^2$
    \begin{align*}
    \frac{\rd [\bvt_i]^2}{\rd t} 
    &= 8[\bvt_i]^2\left(\tat_k[\bvt_i]^2 - \sum_{j\in [d]} \tat_j[\bvt_j]^4 \pm \frac{\n{\Deltat}_F}{|\bvt_i|}\right)
    \le [\bvt_i]^2O\left(\frac{\beta\log d}{d}\right),
    \end{align*}
    where we use $\n{\Deltat}_F=O(\alpha+m\delta_1^2)$ is small enough and $\tat_k\le 1$. 
    This implies $[\bvt_i]^2\le 2c_1\log d/d$ as $t_{max}^\prime=o(d/(\beta \log d))$.
\end{proof}

\subsubsection{Proof of Lemma \ref{lem-phase1-constantgap}}
\lemphaseoneconstantgap*

\begin{proof}
    By Lemma \ref{lem-phase1-remainsmall} we know $[\bvt_i]^2$ will remain $O(\log d /d)$ for those $[\bvx{\tau_0}_i]^2=O(\log d/d)$.

    We now show $[\bvt_k]^2$ will become constant within $t_1^\dbprime$ time. We know $\sum_{i\ne k} \tat_i[\bvt_i]^4\le \beta c_1\log d /d$ for some constant $c_1$. Hence, with the fact $\St_k=\varnothing$, $a_k=\Omega(\beta)$, $[\bvx{\tau_0}_k]^2> \log^4(d)/d$ and $\n{\Deltat}_F=O(\alpha+m\delta_1^2)$,
    \begin{align*}
    \frac{\rd [\bvt_k]^2}{\rd t} 
    &= 8[\bvt_k]^2\left(\tat_k[\bvt_k]^2(1-[\bvt_k]^2) - \sum_{i\ne k} \tat_i[\bvt_i]^4 \pm \frac{\n{\Deltat}_F}{|\bvt_k|}\right)\\
    &\ge 8(1-2c)[\bvt_k]^2 a_k[\bvt_k]^2
    = [\bvt_k]^2 \Omega\left(\frac{\beta \log^4 d}{d}\right).
    \end{align*}
    This implies that within $t_1^\dbprime$ time, we have $[\bvt_k]^2\ge c$. Since $[\bvt_i]^2$ will remain $O(\log d /d)$ for $i\neq k$ and $t\le t_1$, following the same argument above, it is easy to see $\frac{\rd [\bvt_k]^2}{\rd t} \ge 0$ after $[\bvt_k]^2$ reaches $c$. Therefore, $[\bvt_k]^2\ge c$ for $t\le t_1$.
    
\end{proof}

    
    

\subsubsection{Proof of Lemma~\ref{lem-phase1-normgrow}}

\lemphaseonenormgrow*
\begin{proof}
    For $\n{\vt}_2^2$, we have
    \begin{align*}
        \frac{\rd \n{\vt}_2^2}{\rd t} = \n{\vt}^2\left(4\sum_{i\in[d]}\tat_i[\bvt_i]^4\pm \n{\Deltat}_F - 2\lambda\right).
    \end{align*}
    Given the fact $\n{\Deltat}_F=O(\alpha+m\delta_1^2)$ and $\lambda$ are small enough 
    , it is easy to see $\n{\vx{\tau_0^\prime}}_2\ge \delta_0/2$ as $\tau_0^\prime\le t_1$. We now show that there exist time $\tau_1\le t_1^\prime+t_1^\dbprime+t_1^\triprime= t_1$ such that $\n{\vx{\tau_1}}_2\ge\delta_1$. By Lemma \ref{lem-phase1-constantgap} we know $[\bvt_k]^2\ge c$ after time $\tau_0+t_1^\prime\le t_1^\prime+t_1^\dbprime$. And since $\St_k= \varnothing$, we know $\tat_k=a_k=\Omega(\beta)$. Then with the fact that $\n{\Deltat}_F=O(\alpha+m\delta_1^2)$ and $\lambda$ are small enough, we have
    \begin{align*}
        \frac{\rd \n{\vt}^2}{\rd t} \ge \n{\vt}^2 \Omega(\beta).
    \end{align*}
    This implies that $\n{\vx{\tau_1}}_2^2\ge \delta_1^2$ as $t_1^\triprime=\Theta( \log (d/\alpha)/\beta)$. 
\end{proof}


\subsubsection{Proof of Lemma \ref{lem-phase1-normthreshold}}

\lemphaseonenormthreshold*

\begin{proof}
    For $\n{\vt}_2^2$, we have
    \begin{align*}
        \frac{\rd \n{\vt}_2^2}{\rd t} = \n{\vt}^2\left(4\sum_{i\in[d]}\tat_i[\bvt_i]^4\pm \n{\Deltat}_F - 2\lambda\right)
    \end{align*}

    \paragraph{Part 1.}
    By Lemma \ref{lem-phase1-bv4} and $\n{\Deltat}_F=O(\alpha+m\delta_1^2)$, we know
    \begin{align*}
        \frac{\rd \n{\vt}^2}{\rd t} \le \n{\vt}^2 \frac{5\beta c_e \log d}{d}.
    \end{align*}
    This implies $\n{\vt}_2^2=O(\delta_0) $ as $t_1= O(\frac{d}{\beta\log d })$. 
    
    \paragraph{Part 2.}
    By Part 1, we know $\n{\vx{\tau_0}}_2=O( \delta_0)$ and $[\vx{\tau_0}_i]^2=O(\delta_0^2\log d /d)$ for $i\ne k$. For $[\bvx{\tau_0}_i]^2=O(\log d/d)$, we know $[\bvt_i]^2=O(\log d/d)$ for $\tau_0\le t\le \tau_1$ by Lemma \ref{lem-phase1-remainsmall}. We consider following cases separately. 
    \begin{enumerate}
        
        
        
        \item Case 1: Suppose $\hatt_k\le \alpha$ for $t\le \tau_1$. In the following we show there exists some constant $C$ such that for all $i\ne k$ $[\vt_i]^2\le C\delta_0^2 \log d/d$ for $\tau_0\le t \le \tau_1$. Let $\tau_2$ be the first time that the above claim is false, which means for all $i\ne k$ $[\vt_i]^2\le C\delta_0^2 \log d/d$ when $t\le \tau_2$.
    
        For any $i\ne k$, we only need to consider the time period $t\le \tau_2$ whenever $[\vt_i]^2\ge \delta_0^2 \log d/d$. By Lemma~\ref{lem-phase1-calculation}, we have
        \begin{align*}
            \frac{\rd}{\rd t}[\vt_i]^2 
            =&4[\vt_i]^2 \left(2\tat_i [\bvt_i]^2 - \sum_{i\in[d]}\tat_i [\bvt_i]^4 \pm O(\alpha+m\delta_1^2)\right.\\
            &\pm \left.O\left(\frac{(\alpha^2+ d\alpha^3 + d\alpha (1-[\bvt_k]^2)^{1.5}+m\delta_1^2)\n{\vt}}{|\vt_i|}\right)\right)\\
            \le& [\vt_i]^2 \left(O\left(\frac{\beta\log d}{d}\right) + O\left(\frac{(\alpha^2+ \alpha (1-[\bvt_k]^2)^{1.5}+m\delta_1^2)\n{\vt}}{|\vt_i|}\right)\right).
        \end{align*}
        Since for all $i\ne k$ $[\vt_i]^2\le C\delta_0^2 \log d/d$, we know $\sum_{i\ne k}[\vt_i]^2=\n{\vt}^2(1-[\bvt_k]^2)=O(\delta_0^2\log d)$. Together with the fact $[\vt_i]^2\ge \delta_0^2 \log d/d$, we have
        \begin{align*}
            \frac{\rd}{\rd t}[\vt_i]^2 \le [\vt_i]^2 O\left(\frac{\beta\log d}{d}\right).
        \end{align*}
        Since $t_1 = O(d/(\beta\log d))$, we know if we choose large enough $C$, it must be $\tau_2\ge \tau_1$. Therefore, we know for all $i\ne k$ $[\vt_i]^2\le C\delta_0^2 \log d/d$ for $\tau_0\le t \le \tau_1$. Then at time $\tau_1$ when $\n{\vx{\tau_1}}_2\ge\delta_1$, it must be $[\bvt_k]^2\ge 1-\alpha^2$ since $\delta_1=\Theta(\delta_0\log^{1/2} (d)/\alpha)$. 
        
        \item Case 2: We do not make assumption on $\hatt_k$. In the following we show there exists some constant $C$ such that for all $i\ne k$ $[\vt_i]^2\le \delta_1^2\alpha/d$ for $\tau_0\le t \le \tau_1$. Let $\tau_2$ be the first time that the above claim is false, which means for all $i\ne k$ $[\vt_i]^2\le \delta_1^2\alpha/d$ when $t\le \tau_2$.
    
        For any $i\ne k$, we only need to consider the time period $t\le \tau_2$ whenever $[\vt_i]^2\ge \delta_1^2 \alpha/2d$. We have
        \begin{align*}
            \frac{\rd [\vt_i]^2}{\rd t} 
            &= 4[\vt_i]^2 \left(2\tat_i [\bvt_i]^2 - \sum_{i\in[d]}\tat_i [\bvt_i]^4 \pm\frac{\n{\Deltat}_F\n{\vt}_2}{|\vt_i|}\right)\\
            &\le [\vt_i]^2 \left(O\left(\frac{\beta\log d}{d}\right) + O\left(\frac{\alpha+m\delta_1^2}{\alpha^{1/2}d^{-1/2}}\right)\right).
        \end{align*}
        Since $m\delta_1^2=O(\alpha)$ and $t_1 = O(d/(\beta\log d))$, we know it must be $\tau_2\ge \tau_1$. Therefore, we know for all $i\ne k$ $[\vt_i]^2\le \delta_1^2 \alpha/d$ for $\tau_0\le t \le \tau_1$. Then at time $\tau_1$ when $\n{\vx{\tau_1}}_2\ge\delta_1$, it must be $[\bvt_k]^2\ge 1-\alpha$.
        
    \end{enumerate}

\end{proof}

\subsubsection{Proof of Lemma~\ref{lem:above_delta1}}
To prove Lemma~\ref{lem:above_delta1}, we need the following calculation on $\frac{d}{dt} \ns{\vt}.$
\begin{lemma}\label{lem:norm_individual}
Suppose $\vt \in \St_k,$ we have 
\[\frac{d}{dt} \ns{\vt}  = \pr{4\tat_k -2\lambda \pm O(\alpha+m\delta_1^2) }\ns{\vt}.\]
\end{lemma}

\begin{proof}
We can write down $\frac{d}{dt}\ns{\vt}$ as follows:
\begin{align*}
    \frac{d}{dt} \ns{\vt}
    =& \pr{ 4(T^*-\Tt)([\bvt]^{\otimes 4}) -2\lambda} \ns{\vt}\\
    =& \left(4\sum_{i\in[d]}\tat_i[\bvt_i]^4\pm \n{\Deltat}_F - 2\lambda\right)\n{\vt}^2
\end{align*}

Since $[\bvt_k]^2\ge 1-\alpha$, $[\bvt_i]^{2}\leq \alpha$ for any $i\neq k$ and $\n{\Deltat}_F=O(\alpha+m\delta_1^2)$, we have
\begin{align*}
    \frac{d}{dt} \ns{\vt}
    =& \left(4\tat_k - 2\lambda \pm O(\alpha+m\delta_1^2)\right)\n{\vt}^2.
\end{align*}





\end{proof}

Now we are ready to prove Lemma~\ref{lem:above_delta1}.

\lemabovedelta*

\begin{proof}
If $\tat_k = \Omega(\lambda)$ through Phase 1, according to Lemma~\ref{lem:norm_individual}, we know $\ns{\vt}$ will never decrease for any $\vt \in \St_k.$ So, we have $S^{(t_1)}_k\neq \varnothing$ and $\hat{a}_k^{(t_1)}\geq \delta_1^2.$

If $\tat_k = O(\lambda)$ at some time in Phase 1, according to Lemma~\ref{lemma: d hattk, lower bound}, it's not hard to show at the end of Phase 1 we still have $a_k-\hat{a}_k^{(t_1)} = O(\lambda).$ This then implies $\hat{a}_k^{(t_1)} = \Omega(\frac{\epsilon}{\sqrt{d}}).$ Note that we only re-initialize the components that have norm less than $\delta_1.$ As long as $\delta_1^2 = O(\frac{\epsilon}{m\sqrt{d}}),$ we ensure that after the re-initialization, we still have $\hat{a}_k^{(t_1)} = \Omega(\frac{\epsilon}{\sqrt{d}}),$ which of course means $S^{(t_1)}_k\neq \varnothing$.
\end{proof}

\subsubsection{Technical Lemma}

\begin{lemma}\label{lem-phase1-calculation}
    In the setting of Lemma~\ref{lem-phase1-normthreshold}, suppose $\hatt_k\le \alpha$. We have for $i\neq k$
    \begin{align*}
    \frac{\rd}{\rd t}[\vt_i]^2 
    =&4[\vt_i]^2 \left(2\tat_i [\bvt_i]^2 - \sum_{i\in[d]}\tat_i [\bvt_i]^4 \pm O(\alpha+m\delta_1^2)\right.\\ &\pm \left.O\left(\frac{(\alpha^2 + \alpha (1-[\bvt_k]^2)^{1.5}+m\delta_1^2)\n{\vt}}{|\vt_i|}\right)\right).
\end{align*}
\end{lemma}
\begin{proof}
    In order to prove this lemma, we need a more careful analysis on $\frac{d}{dt}[\vt_i]^2$. Recall we can decompose $\Tt$ as $\sum_{i\in[d]}\Tt_i + \Tt_\varnothing$ and further write each $\Tt_i$ as $\hatt_ie_i^{\otimes 4} + (\Tt_i-\hatt_ie_i^{\otimes 4}).$ Note that $\fn{(\Tt_i-\hatt_ie_i^{\otimes 4})}=O(\hatt_i\alpha)$ and $\fn{\Tt_\varnothing}\leq m\delta_1^2.$ We can write down $\frac{\rd}{\rd t}[\vt_i]^2$ in the following form:
    \begin{align*}
    \frac{\rd}{\rd t}[\vt_i]^2 
    =& 4[\vt_i]^2 \left(2a_i [\bvt_i]^2 - \sum_{i\in[d]}a_i [\bvt_i]^4 \right)\\
    &-8\vt_i\n{\vt}\sum_{j\in [d]}\br{\Tt_j([\bvt]^{\otimes 3}, I)}_i
    -8\vt_i\n{\vt}\br{\Tt_\varnothing([\bvt]^{\otimes 3}, I)}_i\\
    &+4\vt_i\sum_{j\in[d]}\br{\Tt_j([\bvt]^{\otimes 4})\vt }_i
    +4\vt_i\br{(\Tt_\varnothing([\bvt]^{\otimes 4})\vt }_i
    \\
    &=4[\vt_i]^2 \left(2a_i [\bvt_i]^2 - \sum_{i\in[d]}a_i [\bvt_i]^4 \right)\\
    &-8\vt_i\n{\vt}\sum_{j\in [d]}\br{\Tt_j([\bvt]^{\otimes 3}, I)}_i
    \pm \vt_i\n{\vt} O(m\delta_1^2)\\
    &+4[\vt_i]^2\sum_{j\in[d]}\Tt_j([\bvt]^{\otimes 4})
    \pm [\vt_i]^2 O(m\delta_1^2)\\
    &=4[\vt_i]^2 \left(2a_i [\bvt_i]^2 - \sum_{i\in[d]}(a_i-\ha_i) [\bvt_i]^4 \pm O(\alpha+m\delta_1^2)\right)\\
    &-8\vt_i\n{\vt}\sum_{j\in [d]}\br{\Tt_j([\bvt]^{\otimes 3}, I)}_i
    \pm \vt_i\n{\vt} O(m\delta_1^2).
    \end{align*}

We now bound the term $\br{\Tt_j([\bvt]^{\otimes 3}, I)}_i$. 
\begin{enumerate}
    \item Case 1: $j=i$. If $\hatt_i= 0$, we know $\Tt_i=0$. Otherwise, denote $x=\inner{\bar{w}_{-i}}{\bvt_{-i}}$, we have 
    \begin{align*}
        &\br{\Tt_i([\bvt]^{\otimes 3}, I)}_i\\
        &= \hatt_i \Et_{i,w} \bar{w}_i\inner{\bar{w}}{\bvt}^3\\
        &= \hatt_i \Et_{i,w} \bar{w}_i\left((\bar{w}_i \bvt_i)^3 +  (\bar{w}_i \bvt_i)^2x + (\bar{w}_i \bvt_i)x^2 + x^3\right)\\
        &\le \hatt_i [\bvt_i]^3 + \hatt_i |\bvt_i| \Et_{i,w}|x| + \hatt_i |\bvt_i| \Et_{i,w}x^2 + \hatt_i \Et_{i,w} x^3.
    \end{align*}
    Since $|x|\le\n{\bar{w}_{-1}}$ and $\Et_{i,w}\n{\bar{w}_{-i}}\le (\Et_{i,w}\n{\bar{w}_{-i}}^2)^{1/2}=O(\alpha)$, we have $\br{\Tt_i([\bvt]^{\otimes 3}, I)}_i = \hatt_i [\bvt_i]^3 + \hatt_i |\bvt_i| O(\alpha) + \hatt_i O(\alpha^{2.5})$.

    \item Case 2: $j=k$. We have $\br{\Tt_k([\bvt]^{\otimes 3}, I)}_i = \hatt_k \Et_{k,w} \bar{w}_i\inner{\bar{w}}{\bvt}^3\le \hatt_k \Et_{k,w} |\bar{w}_i| = O(\alpha^2)$, since $\hatt_k\le \alpha$ and $\Et_{k,w} |\bar{w}_i|\le (\Et_{k,w} |\bar{w}_i|^2)^{1/2}=O(\alpha)$.
    
    \item Case 3: $j\ne i,k$. $j\ne i,k$. If $\hatt_j= 0$, we know $\Tt_j=0$. Otherwise, we can write $\Tt_j$ as $\hatt_j\Et_{j,w}\bar{w}^{\otimes 4}.$ So we just need to bound $\Et_{j,w} \bar{w}_i\inner{\bar{w}}{\bvt}^3.$ We know $\absr{\inner{\bar{w}}{\bvt}} = \absr{\inner{\bar{w}_{-j}}{\bvt_{-j}}+\bar{w}_j\bvt_j }\leq \n{\bar{w}_{-j}}+\sqrt{1-[\bvt_k]^2}.$ So we have 
    \begin{align*}
    \Et_{j,w} \bar{w}_i\inner{\bar{w}}{\bvt}^3 
    =&  \Et_{j,w} \bar{w}_i O\pr{\n{\bw_{-j}}^3+(1-[\bvt_k]^2)^{1.5}}\\
    \leq& O\pr{\alpha^3+\alpha(1-[\bvt_k]^2)^{1.5}},
    \end{align*}
    where in the lase line we use $\Et_{j,w} \bar{w}_i \le (\Et_{j,w} \bar{w}_i^2)^{1/2}=O(\alpha)$.
\end{enumerate} 

Recall that $\tat_i=a_i-\hatt_i$. We now have
\begin{align*}
    \frac{\rd}{\rd t}[\vt_i]^2 
    =& 4[\vt_i]^2 \left(2\tat_i [\bvt_i]^2 - \sum_{i\in[d]}\tat_i [\bvt_i]^4 \pm O(\alpha+m\delta_1^2) \right.\\ &\pm \left.O\left(\frac{(\alpha^2 + \alpha (1-[\bvt_k]^2)^{1.5}+m\delta_1^2)\n{\vt}}{|\vt_i|}\right)\right).
\end{align*}

\end{proof}

\section{Proofs for Phase 2}\label{sec:proof_phase2}
The goal of this section is to show that all discovered directions can be
fitted within time $t_2^{(s)} - t_1^{(s)}$ and the reinitialized components 
will not move significantly. Namely, we prove the following lemma. 
\phasetwomain*

Note that since 
$\delta_1^2 = \poly(\eps) / \poly(d)$ and $\log(d/\eps) = o(d / \log d)$, 
we have $t_2^{(s)} - t_1^{(s)} = \frac{o(d/\log d)}{\beta^{(s)}}$.
\paragraph{Notations}
As in Sec.~\ref{sec: appendix, induction hypothesis}, to simplify the
notations, we shall drop the superscript of epoch $s$, and write 
$\zt := \inner{\bvt}{\bwt}$ and $\tat_k := a_k - \hatt_k$. Within this section,
we write $T := t_2^{(s)} - t_1^{(s)}$.

\paragraph{Proof overview}
The first part is proved using the analysis in Appedix~\ref{sec: appendix, induction hypothesis}. Note that we should view the analysis in this section and the analysis in Appendix~\ref{sec: appendix, induction hypothesis} as a whole induction/continuity argument. It's easy to verify that at any time $t_1^{(s)}\leq t\leq t_2^{(s)}$, Assumption~\ref{assumption: induction, oracle} holds and Proposition~\ref{prop:main} holds.

The second part is a simple corollary of Lemma~\ref{lemma: d hattk, 
lower bound} that gives a lower bound for the increasing speed of $\hatt_k.$

For the third part, we proceed as follows. At the beginning of phase 2, for any 
reinitialized component $\vt$, we know there exists some universal constant
$C > 0$ s.t.~ $[\bvt_k]^2 \le C \log d / d$ for all $k \in [d]$.
Let $T'$ be the minimum time needed for some $[\bvt_k]^2$ to reach 
$2C \log d / d$. For any $t \le T' + t_1^{(s)}$, we have $[\bvt_k]^2 \le 
2C \log d /d$ and then we can derive an upper bound on the movement speed of 
$\vt$, with which we show the change of $[\bvt_k]^2$ is $o(\log d / d)$ within 
time $T$. (Also note this automatically implies that $T' > T$.) To bound the
change of the norm, we proceed in a similar way but with $T'$ being the
minimum time needed for some $\|\vt\|$ to reach $2\delta_0$. (Strictly 
speaking, the actual $T'$ is the smaller one between them.)

\begin{lemma}
  If $S^{(s, t_1^{(s)})}_k 
  \ne \varnothing$, then after at most $\frac{4}{a_k} \log
  \left(\frac{a_k}{2 \delta_1^2}\right)$ 
  time, we have $\tat_k \le \lambda$. 
\end{lemma}
\begin{proof}
  Recall that Lemma~\ref{lemma: d hattk, lower bound} says
  \footnote{$\alpha^2 = o(\lambda)$.}
  \[  
    \frac{1}{\hatt_k} \frac{\rd}{\rd t} \hatt_k
    \ge 2 \tat_k - \lambda - O\left(\alpha^2 \right).
  \]
  As a result, when $\tat_k < 2\lambda / 3$, we have
  $\frac{\rd}{\rd t} \hatt_k \ge \tat_k \hatt_k$ or, equivalently, 
  $\frac{\rd}{\rd t} \tat_k \le - \tat_k \hatt_k$. When $\hatt_k \le 
  a_k / 2$, we have $\frac{\rd}{\rd t} \hatt_k \ge a_k \hatt_k / 2$, whence 
  it takes at most $\frac{2}{a_k}\log\left( \frac{a_k}{2\delta_1^2}\right)$
  time for $\hatt_k$ to grow from $\delta_1^2$ to $a_k / 2$. When 
  $\hatt_k \ge a_k / 2$, we have $\frac{\rd}{\rd t} \tat_k \le - a_k
  \tat_k / 2$, whence it takes at most $\frac{2}{a_k}\log\left(
  \frac{a_k}{2 \lambda}\right)$. Hence, the total amount of time is upper
  bounded by $\frac{2}{a_k}\left( \log\left(\frac{a_k}{2 \delta_1^2}\right) +
  \log\left(\frac{a_k}{2 \lambda}\right) \right)$. Finally, use the fact
  $\lambda > \delta_1^2$ to complete the proof. 
\end{proof}

\begin{lemma}
  \label{lemma: phase 2, bounds for zt}
  For any $k \in [d]$ and $\bvt$ with $\|\bvt\|_\infty^2 \le O(\log d / d)$,
  we have $\Et_{k, w} [\zt]^4 = [\bvt_k]^4 \pm O\left( \frac{\log d}{d} \alpha
  \right)$. Meanwhile, for each $\bwt \in \St_k$, we have
  $\absr{\zt} \le O\left(\sqrt{\frac{\log d}{d}}\right)$. 
\end{lemma}
\begin{proof}
  For simplicity, put $\xt = \inner{\bwt_{-k}}{\bvt_{-k}}$. Then we have
  \begin{align*}
    \Et_{k, w} [\zt]^4
    = \Et_{k, w} \bigg\{ 
        [\bwt_k]^4 [\bvt_k]^4
        & + 4 [\bwt_k]^3 [\bvt_k]^3 \xt
        + 6 [\bwt_k]^2 [\bvt_k]^2 [\xt]^2 \\
      & + 4 \bwt_k \bvt_k [\xt]^3
        + [\xt]^4
      \bigg\}.
  \end{align*}
  For the first term, we have $[\bvt_k]^4 \Et_{k, w} [\bwt_k]^4 
  = [\bvt_k]^4 \left( 1 \pm O( \alpha^2 )\right)$. To bound the rest terms,
  we compute
  \begin{align*}
    \Et_{k, w} \left\{ [\bwt_k]^3 [\bvt_k]^3 \xt \right\}
    &\le O(1) \left( \frac{\log d}{d} \right)^{1.5}
      \Et_{k, w} \sqrt{1 - [\bwt_k]^2}
    \le O(1) \left( \frac{\log d}{d} \right)^{1.5} \alpha, \\
    \Et_{k, w} \left\{ [\bwt_k]^2 [\bvt_k]^2 [\xt]^2 \right\} 
    &\le O(1)\frac{\log d}{d} \alpha^2 \\
    \Et_{k, w} \left\{ \bvt_k [\xt]^3 \right\}
    &\le O(1)\sqrt{\frac{\log d}{d}} \alpha^{2.5} \\
    \Et_{k, w} \left\{ [\xt]^4 \right\}
    &\le O(1)\alpha^3. 
  \end{align*}
  Use the fact $\alpha \le \log d / d$ and we get
  \begin{align*}
    \Et_{k, w} [\zt]^4
    = [\bvt_k]^4 \left( 1 \pm O( \alpha^2 )\right) 
      \pm O(1)\frac{\log d}{d} \alpha
    = [\bvt_k]^4 \pm O\left( \frac{\log d}{d} \alpha \right).
  \end{align*}
  
  For the individual bound, it suffices to note that 
  \[
    \absr{\zt}
    \le \absr{\bvt_k} + \sqrt{1 - [\bwt_k]^2}
    \le O\left(\sqrt{\frac{\log d}{d}}\right)
      + \sqrt{\alpha}
    = O\left(\sqrt{\frac{\log d}{d}}\right).
  \]
\end{proof}

\begin{lemma}[Bound on the tangent movement] 
  In Phase 2, for any reinitialized 
  component $\vt$ and $k \in [d]$, we have $[\bv^{(t_2)}_k]^2 = 
  [\bv^{(t_1)}_k]^2 + o(\log d / d)$. 
\end{lemma}
\begin{proof}
  Recall the definition of $G_1$, $G_2$ and $G_3$ from Lemma~\ref{lemma:
  d vtk2}. By Lemma~\ref{lemma: phase 2, bounds for zt}, we have
  \begin{align*}
    G_1 
    &\le 8 \tat_k \left( 1 - [\bvt_k]^2 \right) [\bvt_k]^4
      + O(1) a_k \frac{\log d}{d} \alpha 
      + 8 \hatt_k \Et_{k, w}\left\{ [\zt]^3 
        \inner{\bw_{-k}}{\bv_{-k}} \right\} \\
    &\le 8 \tat_k \left( 1 - [\bvt_k]^2 \right) [\bvt_k]^4
      + O\left( a_k \frac{\log d}{d} \alpha  \right),
  \end{align*}
  where the second line comes from
  \begin{align*}
    \Et_{k, w}\left\{ [\zt]^3 \inner{\bw_{-k}}{\bv_{-k}} \right\}
    \le O(1) \frac{\log d}{d}
      \Et_{k, w} \sqrt{1 - [\bwt_k]^2}
    \le O\left( \frac{\log d}{d} \alpha \right).
  \end{align*}
  Similarly, we have $|G_2| \le O(1) \sum_{i \ne k} a_i \frac{\log d}{d}\alpha$.
  For $G_3$, by Lemma~\ref{lemma: phase 2, bounds for zt}, we have
  \begin{align*}
    a_i [\bvt_i]^4 - \hatt_i \Et_{i, w} \left\{ [\zt]^4 \right\}
    = \tat_i [\bvt_i]^4 \pm O\left( a_i \frac{\log d}{d} \alpha \right).
  \end{align*}
  Therefore
  \begin{align*}
    |G_3| 
    & \le 8 [\bvt_k]^2 \sum_{i \ne k} \left( 
        \tat_i [\bvt_i]^4 \pm O\left( a_i \frac{\log d}{d} \alpha \right)
      \right) \\
    &\le 8 [\bvt_k]^2 \left( 
        \left(\max_{i \ne k} \tat_i \right) O\left(\frac{\log d}{d}\right)
        + O\left(\frac{\log d}{d} \alpha \right)
      \right) \\
    &\le O\left( \beta^{(s)}  \frac{\log^2 d}{d^2} \right).
  \end{align*}
  Thus\footnote{$\alpha \le O(\beta^{(s)} \log d / d)$},
  \begin{align*}
    \frac{\rd}{\rd t} [\bvt_k]^2
    &\le 8 \tat_k [\bvt_k]^4
      + O\left(\frac{\log d}{d} \alpha  \right)
      + O\left( \beta^{(s)}  \frac{\log^2 d}{d^2} \right) \\
    &\le O\left( \beta^{(s)}  \frac{\log^2 d}{d^2} \right).
  \end{align*}
  Integrate both sides and recall that $T = \frac{o(d/\log d)}{\beta^{(s)}}$.
  Thus, the change of $[\bvt_k]^2$ is $o(\log d / d)$. 
\end{proof}

\begin{lemma}[Bound on the norm growth]
  In Phase 2, for any reinitialized 
  component $\vt$ and $k \in [d]$, we have $\left|\ns{v^{(t_2)}} 
  - \ns{v^{(t_2)}}\right| = o(\delta_0^2)$. 
\end{lemma}
\begin{proof}
  By Lemma~\ref{lemma: d |v|2} and Lemma~\ref{lemma: phase 2, bounds for zt},
  we have
  \begin{align*}
    \frac{1}{2 \ns{\vt}} \frac{\rd}{\rd t} \ns{\vt}
    &\le \sum_{i=1}^d 
      \left( a_i [\bvt_i]^4 - \hatt_i \Et_{i, w} [\zt]^4 \right) \\
    &\le \sum_{i=1}^d 
      \left( \tat_i [\bvt_i]^4 
        + a_i O\left(\frac{\log d}{d} \alpha \right)\right) \\
    &\le \left( \max_{i \in [d]} \tat_i \right)
      O\left( \frac{\log d}{d} \right)
        + O\left(\frac{\log d}{d} \alpha \right) \\
    &= \left( \max_{i \in [d]} \tat_i \right)
      O\left( \frac{\log d}{d} \right).
  \end{align*}
  Recall that $\max_{i \in [d]} \tat_i \le O(\beta^{(s)})$ and
  $\|\vt\| \le O(\delta_0)$. Hence,
  \[
    \frac{\rd}{\rd t} \ns{\vt}
    \le O\left( \beta^{(s)} \frac{\log d}{d} \right) \delta_0^2.
  \]
  Integrate both sides, use the fact $T = \frac{o(d/\log d)}{\beta^{(s)}}$, and
  then we complete the proof.
\end{proof}

\begin{proofof}{Lemma~\ref{lem:phase2}}
Lemma~\ref{lem:phase2} follows by combining the above lemmas with the analysis in Appendix~\ref{sec: appendix, induction hypothesis}.
\end{proofof}

\section{Proof for Theorem~\ref{thm:main}}\label{sec:proof_main_theorem}
In the section, we give a proof of Theorem~\ref{thm:main}. 
\maintheorem*

Note that Proposition~\ref{prop:main} guarantees any ground truth component with $a_i\geq \beta^{(s)}/(1-\gamma)$ must have been fitted before epoch $s$ starts. When $\beta^{(s)}$ decreases below $O(\epsilon/\sqrt{d}),$ all the ground truth components larger than $O(\epsilon/\sqrt{d})$ have been fitted and the residual $\fn{T-T^*}$ must be less than $\epsilon.$ Since $\beta^{(s)}$ decreases in a constant rate, the algorithm must terminate in $O(\log(d/\epsilon))$ epochs.

\begin{proof}
According to Lemma~\ref{lem:phase1} and Lemma~\ref{lem:phase2}, we know Proposition~\ref{prop:main} holds through the algorithm.
We first show that $\beta^{(s)}$ is always lower bounded by $\Omega(\epsilon/\sqrt{d})$ before the algorithm ends. For the sake of contradiction, assume $\beta^{(s)}\leq O(\frac{\epsilon}{\sqrt{d}})$. We show that $\fn{T^{(s,0)}-T^*}<\epsilon,$ which is a contradiction because our algorithm should have terminated before this epoch. For simplicity, we drop the superscript on epoch $s$ in the proof. 

We can upper bound $\fn{T^* - \Tt}$ by splitting $T^*$ into $\sum_{i\in [d]}T^*_i$ and splitting $\Tt$ into $\sum_{i\in [d]}\Tt_i + \Tt_\varnothing.$ Then, we have 
\begin{align*}
    \fn{T^*-\Tt} \leq& \fn{\sum_{i\in d}(a_i-\hatt_i)e_i^{\otimes 4}} + \sum_{i\in [d]}\fn{\Tt_i - \hatt_i e_i^{\otimes 4}} + \fn{\Tt_\varnothing}\\
    \leq& O\pr{\sqrt{d}\max\pr{\beta^{(s)},\lambda}} + O(\alpha + m\delta_1^2),
\end{align*}
where the second inequality holds because $(a_i-\hatt_i)\leq O(\max\pr{\beta^{(s)},\lambda}), \fn{\Tt_i - \hatt_i e_i^{\otimes 4}}\leq O(\hatt_i \alpha)$ and $\fn{\Tt_\varnothing}\leq m\delta_1^2.$
Choosing $\lambda,\alpha = O(\frac{\epsilon}{\sqrt{d}})$ and $\delta_1^2 = O(\frac{\epsilon}{m\sqrt{d}}),$ we have 
\[\fn{T^*-\Tt}<\epsilon.\]

Since $\beta^{(s)}$ starts from $O(1)$ and decreases by a constant factor at each epoch, it will decrease below $O(\frac{\epsilon}{\sqrt{d}})$ after $O(\log(d/\epsilon))$ epochs. This means our algorithm terminates in $O(\log(d/\epsilon))$ epochs.
\end{proof}

\section{Experiments}\label{sec:experiment}
In Section~\ref{sec:exp_detail}, we give detailed settings for our experiments in Figure~\ref{fig:ortho}. Then, we give additional experiments on non-orthogonal tensors in Section~\ref{sec:exp_nonortho}.

\subsection{Experiment settings for orthogonal tensor decomposition}\label{sec:exp_detail}
We chose the ground truth tensor $T^*$ as $\sum_{i\in [5]} a_i e_i^{\otimes 4}$ with $e_i\in \R^{10}$ and $a_i/a_{i+1} = 1.2.$ We normalized $T^*$ so its Frobenius norm equals $1$. 

Our model $T$ was over-parameterized to have $50$ components. Each component $W[:,i]$ was randomly initialized from $\delta_0 \text{Unif}(\mathbb{S}^{d-1})$ with $\delta_0 =10^{-15}.$

The objective function is $\frac{1}{2}\fns{T-T^*}.$ We ran gradient descent with step size $0.1$ for $2000$ steps. We repeated the experiment from $5$ different experiments and plotted the results in Figure~\ref{fig:ortho}. Our experiments was ran on a normal laptop and took a few minutes. 

\subsection{Additional results on non-orthogonal tensor decomposition}\label{sec:exp_nonortho}
In this subsection, we give some empirical observations that suggests non-orthogonal tensor decomposition may not follow the greedy low-rank learning procedure in~\cite{li2020towards}.

\paragraph{Ground truth tensor $T^*$:} The ground truth tensor is a $10\times 10\times 10\times 10$ tensor with rank $5$. It's a symmetric and non-orthogonal tensor with $\fn{T^*}=1.$ The specific ground truth tensor we used is in the code.

\paragraph{Greedy low-rank learning (GLRL):} We first generate the trajectory of the greedy low-rank learning. In our setting, GLRL consists of $5$ epochs. At initialization, the model has no component. At each epoch, the algorithm first adds a small component (with norm $10^{-60}$) that maximizes the correlation with the current residual to the model, then runs gradient descent until convergence. 

To find the component that has best correlation with residual $R$, we ran gradient descent on $R(w^{\otimes 4})$ and normalize $w$ after each iteration. In other words, we ran projected gradient descent to solve $\min_{w\mid \n{w}=1}R(w^{\otimes 4}).$ We repeated this process from $50$ different initializations and chose the best component among them.

In the experiment, we chose the step size as $0.3$. And at the $s$-th epoch, we ran $s\times 2000$ iterations to find the best rank-one approximation and also ran $s\times 2000$ iterations on our model after we included the new component. After each epoch, we saved the current tensor as a saddle point. We also included the zero tensor as a saddle point so there are $6$ saddles in total. 

Figure~\ref{fig:greedy_loss} shows that the loss decreases sharply in each epoch and eventually converges to zero.

\begin{figure}[h]
\centering
\subfigure{
    \includegraphics[width=3in]{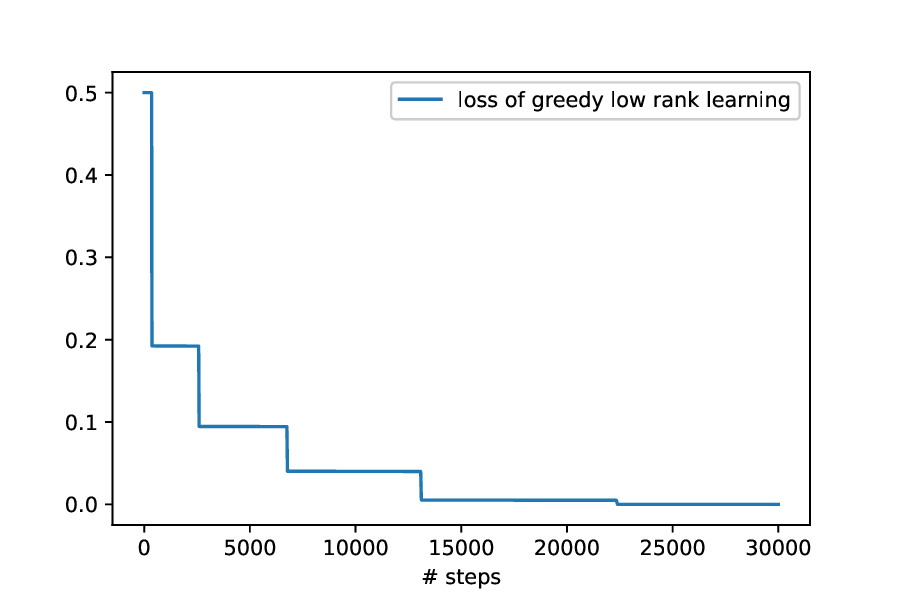}
}
\caption{Loss trajectory of greedy low-rank learning. }
\label{fig:greedy_loss}
\end{figure}

\paragraph{Over-parameterized gradient descent:}
If the over-parameterized gradient descent follows the greedy low-rank learning procedure, one should expect that the model passes the same saddles when the tensor rank increases. To verify this, we ran experiments with gradient descent and computed the distance to the closest GLRL saddles at each iteration. 

Our model has $50$ components and each component is initialized from $\delta_0 \text{Unif}(\mathbb{S}^{d-1})$ with $\delta_0 = 10^{-60}.$ We ran gradient descent with step size $0.3$ for $1000$ iterations.

\begin{figure}[h]
\centering
\subfigure{
    \includegraphics[width=2.65in]{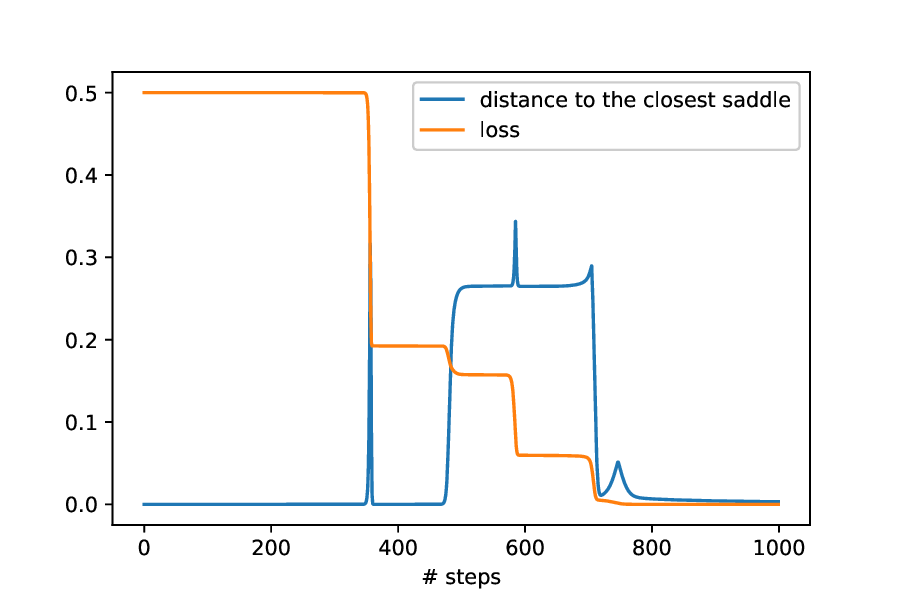}
}
\subfigure{
    \includegraphics[width=2.65in]{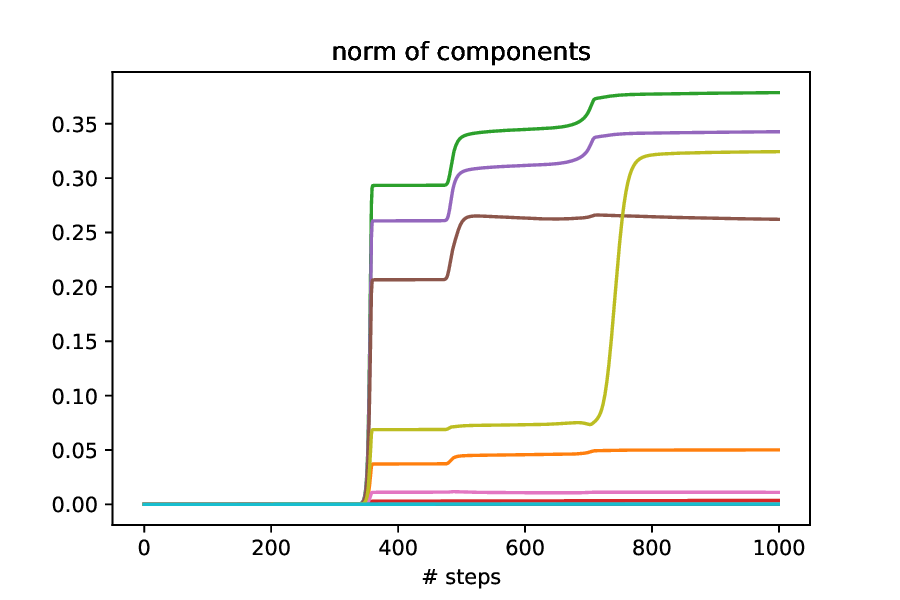}
}
\caption{Non-orthogonal tensor decomposition with number of components $m=50$ and initialization scale $\delta_0=10^{-60}.$
The left figure shows the loss trajectory and the distance to the closest GLRL saddles; the right figures shows the norm trajectory of different components.}
\label{fig:distance_norm_growth}
\end{figure}

Figure~\ref{fig:distance_norm_growth} (left) shows that after fitting the first direction, over-parameterized gradient descent then has a very different trajectory from GLRL. After roughly $450$ iterations, the loss continues decreasing but the distance to the closest saddle is high. After $800$ iterations, gradient descent converges and the distance to the closest saddle (which is $T^*$) becomes low. 

In Figure~\ref{fig:distance_norm_growth} (right), we plotted the norm trajoeries for $10$ of the components. The figure shows that some of the already large components become even larger at roughly $450$ iterations, which corresponds to the second drop of the loss. We picked two of these components and found that their correlation $\inner{\bar{w}}{\bar{v}}$ drops from $1$ at the $400$-th iteration to $0.48$ at the $550$-th iteration. This suggests that two large component in the same direction can actually split into two directions in the training. 

One might suspect that this phenomenon would disappear if we use more aggressive over-parameterization and even smaller initialization. We then let our model have $1000$ components and let the initialization size to be $10^{-100}$ and re-did the experiments. We observed almost the same behavior as before. Figure~\ref{fig:distance_norm_growth_large} (left) shows the same pattern for the distance to closest GLRL saddles as in Figure~\ref{fig:distance_norm_growth}. In Figure~\ref{fig:distance_norm_growth_large} (right), we randomly chose $10$ of the $1000$ components and plotted their norm change, and we again observe that one large component becomes even larger at roughly iteration 700 that corresponds to the second drop of the loss function.

\begin{figure}[h]
\centering
\subfigure{
    \includegraphics[width=2.65in]{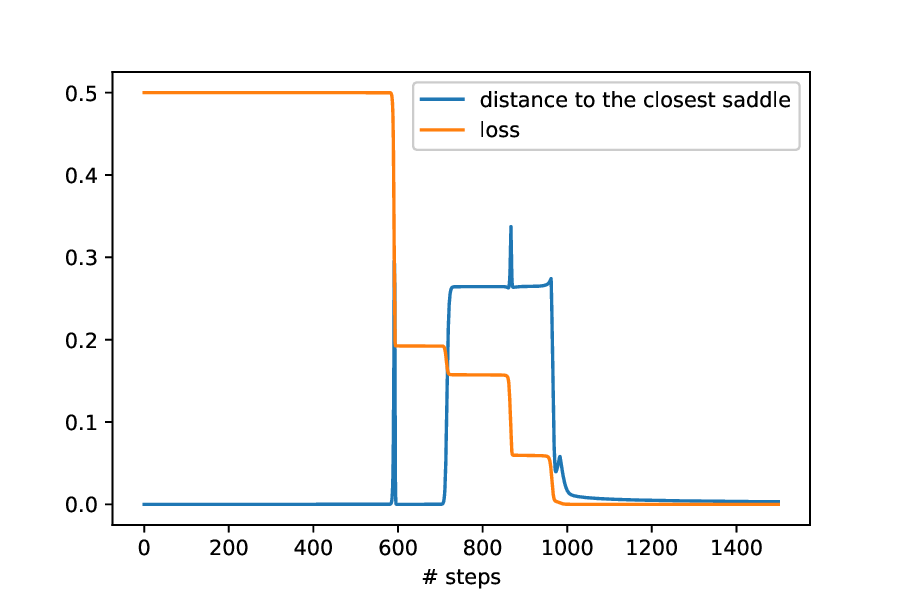}
}
\subfigure{
    \includegraphics[width=2.65in]{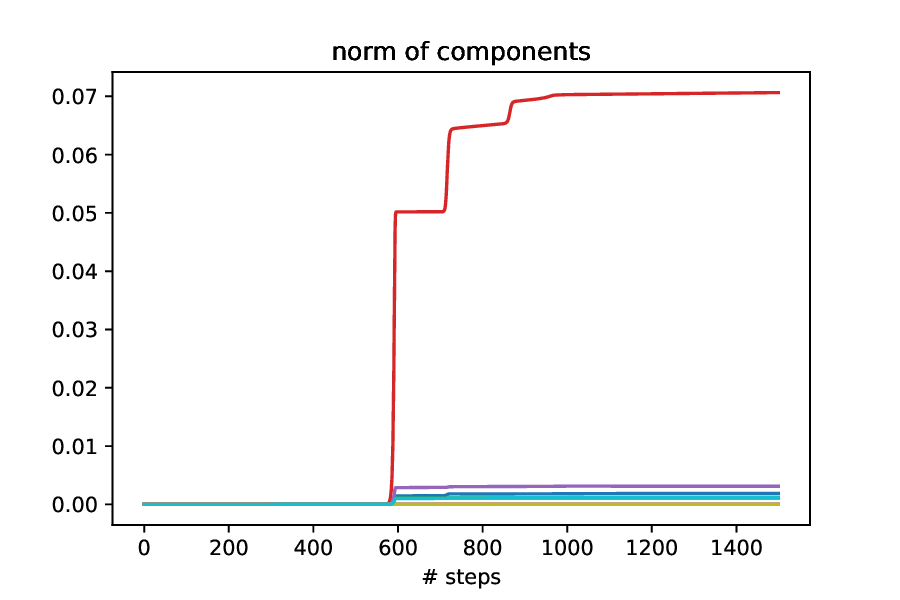}
}
\caption{Non-orthogonal tensor decomposition with number of components $m=1000$ and initialization scale $\delta_0=10^{-100}.$
The left figure shows the loss trajectory and the distance to the closest GLRL saddles; the right figures shows the norm trajectory of different components.}
\label{fig:distance_norm_growth_large}
\end{figure}

\end{document}